\documentclass{article}
\usepackage{color,xcolor}

\usepackage{url}
\usepackage[colorlinks,linkcolor=magenta,citecolor=blue, pagebackref=true,backref=true]{hyperref}

\usepackage[english]{babel}
\usepackage[utf8x]{inputenc}
\usepackage[T1]{fontenc}
\usepackage{physics}
\long\def\comment#1{}

\usepackage[a4paper,top=3cm,bottom=2cm,left=3cm,right=3cm,marginparwidth=1.75cm]{geometry}

\usepackage{lipsum, babel}
\usepackage{amsmath}
\usepackage{graphicx}
\usepackage{amsthm}
\usepackage{amssymb}
\usepackage{algorithm}
\usepackage{algorithmic}
\usepackage{bbm}
\usepackage[colorinlistoftodos]{todonotes}
\usepackage{caption}
\usepackage{float}
\usepackage{subcaption}
\usepackage{multirow}
\usepackage{makecell}
\usepackage{bigints}
\usepackage{algorithmic}
\usepackage{algorithm}
\usepackage{pdflscape}
\usepackage{comment}
\usepackage{color}
\newcommand{\countV}[1]{\#\left(#1\right)}
\newcommand{\Fcal}{\ensuremath{\mathcal{F}}}
\newcommand{\Ocal}{\ensuremath{\mathcal{O}}}
\newcommand{\Ccal}{\ensuremath{\mathcal{C}}}
\newcommand{\Pcal}{\ensuremath{\mathcal{P}}}

\newcommand{\Fbb}{\ensuremath{\mathbb{F}}}
\newcommand{\Fbbsim}{\ensuremath{\Fbb_2^2(\mu)/\sim}}

\theoremstyle{plain}
\newtheorem{definition}{Definition}
\numberwithin{definition}{section}

\newtheorem{proposition}{Proposition}
\numberwithin{proposition}{section}
\newtheorem{lemma}{Lemma}

\long\def\comment#1{}

\begin{document}
\begin{center}
{\bf{\LARGE{Robust Unsupervised Learning of Temporal Dynamic Interactions}}}

\vspace*{.2in}
{\large{
\begin{tabular}{ccc}
Aritra Guha$^{ \ddagger}$ &  Rayleigh Lei$^{\ddagger}$ & Jiacheng Zhu$^{ \dagger}$ \\
\end{tabular}
\begin{tabular}{cc}
      XuanLong Nguyen$^{\ddagger}$ & Ding~Zhao$^{\dagger}$ 
\end{tabular}
}}

\vspace*{.2in}

\begin{tabular}{c}
Department of Mechanical Engineering, Carnegie Mellon University$^\dagger$ \\
Department of Statistics, University of Michigan$^\ddagger$ \\
\end{tabular}

\vspace*{.2in}

\today

\vspace*{.2in}

\begin{abstract} 
Robust representation learning of temporal dynamic interactions is an important problem in robotic learning in general and automated unsupervised learning in particular. Temporal dynamic interactions can be described by (multiple) geometric trajectories in a suitable space over which unsupervised learning techniques may be applied to extract useful features from raw and high-dimensional data measurements. Taking a geometric approach to robust representation learning for temporal dynamic interactions, it is necessary to develop suitable metrics and a systematic methodology for comparison and for assessing the stability of an unsupervised learning method with respect to its tuning parameters. 
Such metrics must account for the (geometric) constraints in the physical world as well as the uncertainty associated with the learned patterns. In this paper we introduce a model-free metric based on the Procrustes distance for robust representation learning of interactions, and an optimal transport based distance metric for comparing between distributions of interaction primitives. These distance metrics can serve as an objective for assessing the stability of an interaction learning algorithm. They are also used for comparing the outcomes produced by different algorithms. Moreover, they may also be adopted as an objective function to obtain clusters and representative interaction primitives. These concepts and techniques will be introduced, along with mathematical properties, while their usefulness will be demonstrated in unsupervised learning of vehicle-to-vechicle interactions extracted from the Safety Pilot database, the world's largest database for connected vehicles.
\end{abstract}
\end{center}

\section{Introduction}

Advances in large scale data processing and computation enables the application of sophisticated learning algorithms to robotic design in complex and dynamic environments. In many applications a fundamental challenge lies not only in learning about the interaction between the ego agent and the environment, but also interactions between multiple agents. Due to the high dimensionality and typically noisy nature of the data required for such learning tasks,  a standard approach is to utilize strong modeling assumptions on the interactions. For example, the interaction between a robotic agent and the environment can be represented by instantaneous physical variables such as positions, velocities, a time series of which are then endowed with a stationary distribution for mathematical convenience and interpretability (e.g., via a Markov process framework). While such approach is useful in highly controlled environments, the strong modeling assumption are usually violated in domains where the interactions among agents and with the environment are highly dynamic~\cite{foerster2017multi-agent}. Such domains require the development of more robust and data-driven representation learning approaches.

As a concrete example which serves as a primary motivation for this work, take the interaction between two intelligent vehicles that approach each other in a typical intersection. What the two vehicles proceed to do next depend on what they can learn of their encounter in real-time. The two cars may come toward the intersection in varying speeds at perhaps slightly different time points. They may or may not signal their intention. For example, one plans to go straight while the other plans to take a turn cutting through the other's path. Not only do the two agents have to learn their temporally varying interaction, they have to do so quickly and accurately while continually negotiating the traffic. In this type of applications where the interaction is highly dynamic, a promising approach to robust interaction learning is by decomposing the interaction in terms of simpler elements~\cite{1549935, Wang-Zhao-2017}. For traffic applications, such interaction elements are called traffic \emph{primitives}. These primitives can be learned, labeled, and effectively utilized for subsequent tasks such as vehicle trajectory prediction~\cite{zhu2019_nplstm_trajectory}, traffic data generation \cite{ding2018multi_vae_generation, zhang2019multi_gp_generation}, or anomaly detection \cite{ zhang2019learning_primitve_itsc}.

Stripping away the language of vehicle-to-vehicle (V2V) interactions, the temporal dynamic interaction between two agents comprises of a pair of well-aligned trajectories defined on a suitable space that satisfy constraints presented by the environment and agents' behaviors. Thus, the goal of robust representation learning of a pairwise interaction between the two dynamic agents boils down to the learning of pairs of functions or curves which describe the aligned car physical movements and/or driving behaviors. Such a mathematical viewpoint can be generalized to interactions among three or more vehicles. In this paper we will focus on the learning of interactions in two-agent dynamic scenarios. Although our work is motivated by the learning of multi-agent traffic interaction's primitives, we believe that the techniques developed here can be utilized to other settings of multi-agent temporal dynamic interaction learning.

Within the context of real-world traffic learning, both rule-based methods \cite{frazzoli2005maneuver_primitive}, supervised learning \cite{pervez2017learning_primitive_supervised}, and unsupervised learning \cite{Wang-Zhao-2017} have been applied to identify the interaction primitives. Due to the heterogeneity and complexity of traffic scenarios, unsupervised learning is a powerful tool to identify latent structures in unlabeled traffic scenario time series data; the goal is to organize the data into homogeneous groups/ clusters \cite{Bender-et-al-2015, Hamada-et-al-2016, Taniguchi-et-al-2015, Wang-Xi-Zhao-2017, liao2005clustering_survey}. Within automatically learned clusters, interpretable and typical driving behaviors can be obtained and analyzed, e.g., left/right turns along with multiple attributes including speed, acceleration, yaw rate and side-slip angle using Dynamic Time Wrapping (DTW) as a similarity measure \cite{Yao_2019_meta_learning_traffic}. Statistical model-based approaches that can learn complex driving behaviors while allowing for encoding domain-knowledge are also available. For instance, primitive segments extracted from time series traffic data can be obtained without specifying the number of categories via Bayesian nonparametric methods based on Dirichlet processes. They include hierarchical Dirichlet Process Hidden Markov Model (HDP-HMM) \cite{Taniguchi-et-al-2016, Wang-Xi-Zhao-2017}. Dirichlet process mixtures of Gaussian processes were also successfully employed to identify complex multi-vehicle velocity fields \cite{guo2019modeling_dpgp, joseph2011bayesian_dpgp}.

Given the plethora of methods and the need for learning complex interaction patterns in dynamic domains, it is natural to ask which method one should use.
For unsupervised learning, this question is particularly challenging because one typically works with unlabelled data and without immediately available objective functions for the quality of learned clusters of interactions, especially ones which are mathematically represented as a collection of two or more curves taking values in a suitable space, as discussed above. In addition, while the problem of devising techniques which are free of any tuning parameters is an important one, parameter-free algorithms tend to be not robust. A typical unsupervised learning method still requires some prior knowledge or pre-defined parameters (tuning knobs). As a result, clustering results may still be sensitive to these choices. Thus, even when a method is settled on, it is still an important issue how to handle the various tuning knobs and to assess their sensitivity, or stability with respect to changes in the tuning parameters. 

Identifying suitable clustering criteria and analyzing learning stability/sensitivity have received much attention in data mining and statistical learning literatures.
For clustering criteria, there are broadly two categories: internal and external criteria \cite{XuWunschClustering2009}. Internal criteria relies on a similarity or dissimilarity measure that may be applied to the data samples. Such measures evaluate how alike the members of the same cluster are, how different the members of different clusters are, or some combination of thereof \cite{RokachMaimonClusteringMethods2005,XuTianComprehensiveSurveyClustering2015}. 
On the other hand, there is a priori structure how the data should be partitioned, external criteria allow one to compare the clustering results against this structure \cite{RokachMaimonClusteringMethods2005}. Examples include Rand index, mutual information and model-based likelihood-type objectives \cite{XuTianComprehensiveSurveyClustering2015}.

Meanwhile, there is a rich literature on sensitivity analysis that focuses on the impacts of changes in model/method specification on the learning outcomes, see, e.g. textbooks \cite{ChatterjeeHadiSensitivityAnalysisLinear1988, SaltelliEtAlSensitivityAnalysisPractice2004, SaltelliEtAlGlobalSensitivityAnalysis2008}. If we focus on Bayesian methods or model-based methods, the key issue is on the effect of the prior/ model specification. 
While there are a number of variations, most sensitivity analysis techniques involve model fitting with varying prior/ model specifications, and assessing the impacts on posterior distributions or estimates of parameters of interest. A model is said to be robust if the estimates are relatively insensitive to such varying specifications \cite{GustafsonLocalRobustnessBayesian2000,SivaganesanGlobalLocalRobustness2000}. Alternatively, instead of varying the model parameters one may consider perturbing data: a geometric framework was developed to conduct sensitivity analysis with respect to the perturbation of the data, the prior and the sampling distribution for a class of statistical models. Within this framework, various geometric quantities were studied to characterize the intrinsic structure and effect of the perturbation \cite{ZhuEtAlBayesianInfluenceAnalysis2011}.

\comment{
Briefly, there are three broad categories: informal sensitivity analysis, global sensitivity analysis, and local sensitivity analysis \cite{GustafsonLocalRobustnessBayesian2000}. In an informal sensitivity analysis, one might fit the model with a few different priors and compare the results. A global sensitivity approach formalizes this by considering a class of priors and then examining the range of results \cite{SivaganesanGlobalLocalRobustness2000}. 
If the class of priors is reasonably sized and the resultant range of posterior learning is small, then the model is said to be robust, i.e., insensitive to the prior \cite{SivaganesanGlobalLocalRobustness2000}. Alternatively, one may consider the rate of change in the posterior learning results based on a local change in prior, where the rate of change is measured via a suitable notion of derivative of the posterior distribution, an approach known as local sensitivity analysis \cite{GustafsonLocalRobustnessBayesian2000}. 
This type of analysis has been applied to hierarchical models \cite{RoosEtAlSensitivityAnalysisBayesian2015}. The variational Bayes technique was also employed to approximate the change in the posterior distribution on these types of models \cite{GiordanoEtAlCovariancesRobustnessVariational2018}. Finally, a geometric framework was developed to conduct sensitivity analysis with respect to the perturbation of the data, the prior and the sampling distribution for a class of statistical models. Within this framework, various geometric quantities were studied to characterize the intrinsic structure and effect of the perturbation \cite{ZhuEtAlBayesianInfluenceAnalysis2011}.}

To assess the quality of unsupervised learning methods for temporal dynamic interactions, at a high level one may consider the aforementioned methods and frameworks. Moreover, it is necessary to develop a set of suitable metrics for interaction comparison and for assessing the stability of an unsupervised learning method with respect to its tuning parameters. 
Motivated by the representation of dynamic vehicle-to-vehicle interactions that arise in the traffic learning domain, one has to effectively deal with pairs of aligned functions, i.e., trajectories taking values in a suitable space, which is typically non-Euclidean and has high or infinite dimensions. Such metrics must account for the geometric constraints in the physical world as well as the uncertainty associated with the learned patterns. 

To this end, we introduce a model-free metric on pairs of functions based on a Procrustes-type distance, and an optimal transport based Wasserstein distance metric for comparing between distributions of such pairs of functions. The former metric is critical because it preserves translation and rotation invariance, key properties required for capturing the essence of the temporal dynamic between two autonomous or semi-autonomous agents (e.g., vehicles or robots). The latter metric is also appropriate because the result of a clustering algorithm can be mathematically represented as the solution of an optimal transport problem~\cite{Graf-Luschgy-00,ho2017multilevel}.
In addition to some connection to optimal transport based clustering, it is worth noting how
our technical contributions are also inspired by several other prior lines of work. 
In particular, Procrustes-type metrics have been employed in generalized Procrustes analysis which solves the problem of reorienting points to a fixed configuration~\cite{Gower-75-procrustes}. Similar metrics have also been successfully used in literature to study such problems of shape preservation~\cite{Srivastava-Shape-16} as well for alignment of manifolds~\cite{procrustes_manifold,continuous-procrustes-13}. 
In our work, we use it to solve the clustering problem by comparing pairs of curves, each of which may be viewed as manifolds on $\mathbb{R}^2$.  

Finally, we note that the introduced distance metrics can serve as an objective for assessing the stability of an interaction primitive learning algorithm. They are also used for comparing the outcomes produced by different algorithms. Furthermore, they may also be adopted as an objective function to obtain clusters of interactions, and the representative interactions. These concepts and techniques will be introduced in this paper, along with mathematical properties, while their usefulness will be demonstrated in the analysis of vehicle-to-vehicle interactions that arise in the Safety Pilot database \cite{bezzina2014safety}, the world's largest database for connected vehicles.

The paper is organized as follows. In Section~\ref{Section:distance}, we describe a distance metric for pairs of trajectories and explicate its useful mathematical properties. Building on this, Section~\ref{Section: Distribution_primitives} studies distributions of trajectory pairs, which lead to methods for obtaining and assessing clusters of interactions. Finally, Section~\ref{Section: Results} illustrates our methods on the clustering analysis of vehicle-to-vehicle interactions data. 

\section{A distance metric on temporal interactions}
\label{Section:distance}

Because a temporal interaction between two agents is composed of trajectories, we need to first formally define a trajectory. Let $f:\mathbb{R}\rightarrow \mathbb{R}^2$ denote a trajectory of an object (e.g., vehicles, robots). In particular, $f(t)$ represents the location of the object at time-point $t$. 
It suffices for our purpose to restrict to $t\geq 0$.

We can consider all possible trajectories in a similar manner. Define the set of all possible trajectories as $\Fbb=\{f: [0,\infty) \rightarrow \mathbb{R}^2: f \text{ is continuous}\}$. The set of all possible trajectories up to time-point $t$  starting from time-point $s$ is denoted by $\Fcal_{[s,t)}=\{f:[s,t)\rightarrow \mathbb{R}^{2}| f\in \Fbb\}$. Similarly for $ (t_1,\dots,t_k) \in \mathbb{R}_+^{k}$ we will use $\Fcal_{t_1,\ldots,t_k} :=\{ (f (t_1),\ldots,f (t_k)): f \in \Fbb\}$. Also, we define $\Fcal:= \cup_{s,t \in \mathbb{R}^+}\Fcal_{[s,t)}$.

Next, operations can be defined on these trajectories. For any $c\in \mathbb{R}^2$, and $f \in \Fbb$ we define $f+c\in \Fbb$ as $ (f+c) (x)=f (x) + c$ for all $x \in [0,\infty)$. Similarly, for any orthogonal matrix $O \in \mathbb{R}^{2 \times 2}$, define the function $O\odot f \in \Fbb$ as $ (O\odot f)  (x)=O \cdot f(x)$ for all $x \in [0, \infty)$, where $O \cdot f  (x)$ is the usual matrix product between matrix $O$ and vector $f  (x)$ which have matching dimensions. 

With these definitions in place, we now define an interaction and operations on these interactions. An interaction 
is an ordered pair $ (f_1, f_2)$ such that $f_1, f_2 \in \Fbb$. We also define operations on interactions as well. Let $SO (n)$ be the group of $n \times n$ orthogonal matrices with determinant $+1$. For a pair $f_1,f_2 \in \Fbb$, we define $\Ocal_{ (f_1, f_2)}=\{  (O \odot f_1, O \odot f_2) : O \in SO (2)\}$. Similarly, define $\Ccal_{ (f_1,f_2)}=\{ (f_1+c,f_2+c) : c \in \mathbb{R}^{2}\}$.

\subsection{Rotation and translation-invariant metrics on curves}
\label{ssection:rotation}

To evaluate the stability and overall quality of clustering, we want a distance metric, $d:\Fbb^2 \times \Fbb^{2} \rightarrow \mathbb{R}_{+}$, where  ($\Fbb^{2} =\Fbb \times \Fbb$), that has the following properties: 
\begin{itemize}
    \item[ (a)] Distance between two interactions is invariant with respect to the re-ordering of corresponding trajectories, i.e., for $f_{11},f_{12},f_{21},f_{22} \in \Fbb$, the following holds:
    \begin{eqnarray}
    d ( (f_{11},f_{12}), (f_{21},f_{22}))= d ( (f_{12},f_{11}), (f_{21},f_{22})). \nonumber
    \end{eqnarray}
    \item[ (b)] Distance between a pair of interactions is invariant of starting points of the trajectories composing the interactions, given the knowledge of the relative distance of the starting points of trajectories comprising each interaction. Specifically, if $ (f'_1,f'_2) \in \Ccal_{ (f_1,f_2)} \cup \Ocal_{ (f_1,f_2)}$, then, 
    \begin{eqnarray}
    d ( (f'_1,f'_2), (f_1,f_2))=0. \nonumber
    \end{eqnarray}
\end{itemize} 

Condition (a) enables the removal of order in a pair of curves in an interaction, while condition  (b) in essence characterizes \textbf{rotational} and \textbf{translational invariance} of interactions. We will henceforth use $ (\Ocal,\Ccal)_{ (f_1,f_2)}$ to denote the set $\{ (O \odot f_1 + c, O \odot f_2+ c) : O \in SO (2), c\in \mathbb{R}^2\}$. As shown in Lemma~\ref{lemma:invariance}, condition (b) implies that $d ( (f_{11},f_{12}), (f_{21},f_{22}))=d ( (O \odot f_{11} +c, O \odot f_{12} +c), (f_{21},f_{22}))$ for all $c \in \mathbb{R}^2$, $O \in SO (2)$. 
This appears to be a reasonable requirement since the exact location and orientation of interactions should not affect the classification of different interactions into clusters characterized by "primitives". Note that throughout this paper we only consider non-reflective rotational transforms, i.e., transforms involving orthogonal matrices, O, such that $\textrm{det}(O)=+1$.

Let $\rho$ be a distance metric for $\Fbb^2$. We will then construct a metric $d$ satisfying (a) and (b) from $\rho$. Definition \ref{def:rotation_translation_invariant_metric} shows how we can define $d$ in terms of $\rho$.

\begin{definition}
\label{def:rotation_translation_invariant_metric}

Define Procrustes distance
\begin{eqnarray}
\label{eq:rotation_translation_invariance}
& &  d ( (f_{11},f_{12}), (f_{21},f_{22})) \\ 
 &:=&\inf_{ (f'_1,f'_2) \in  (\Ocal,\Ccal)_{ (f_{21},f_{22})}} \biggr \{ \min \biggr \{ \rho ( (f_{11},f_{12}), (f'_1,f'_2)),  \rho ( (f_{12},f_{11}), (f'_1,f'_2)) \biggr \} \biggr \}.\nonumber
\end{eqnarray}
\end{definition}

From the definition of metric $d$ above, it is clear that  $ (f_{21},f_{22}) \in  (\Ocal,\Ccal)_{ (f_{11},f_{12})} \cup (\Ocal,\Ccal)_{ (f_{12},f_{11})} \iff d ( (f_{11},f_{12}), (f_{21},f_{22}))=0$. With that knowledge, we can define an equivalence relation, $\sim$,  as 
\begin{eqnarray}
\label{eq:metric_quotient_space}
  (f_{11},f_{12}) \sim  (f_{21},f_{22}) \iff d ( (f_{11},f_{12}), (f_{21},f_{22}))=0.
\end{eqnarray}

Although $d$ is not a proper metric on $\Fbb^2$, as  Proposition \ref{proposition:proper_metric} shows,  $d$ does define a metric on the quotient space relative to the equivalence relation.

\begin{proposition}
\label{proposition:proper_metric}
Let $\rho$ be a distance metric on $\Fbb^2$ such that for all $f_{11},f_{12},f_{21},f_{22} \in \Fbb$,
\begin{enumerate}
    \item[(i)] $\rho$ satisfies, for some function $h$,
\begin{eqnarray}
\rho ( (f_{11},f_{12}), (f_{21},f_{22}))= h( (f_{11},f_{12})-(f_{21},f_{22})). \nonumber
\end{eqnarray}

\item[(ii)] $\rho $ is an inner-product norm.
\end{enumerate}
Then $d$ given by Eq. \eqref{eq:rotation_translation_invariance} is a distance metric on the quotient space $\Fbb^2/ \sim$.
\end{proposition}

Proposition \ref{proposition:proper_metric} also provides a method to build a metric satisfying conditions (a) and (b) above. One way to do so is from a probability measure perspective. 
In fact, let $\mu$ be a probability measure on $[0,\infty)$. We consider the set of trajectories with integrable Euclidean norm on $[0,\infty)$, i.e., we restrict attention to the following set of trajectories:
\begin{eqnarray}
\Fbb_2 (\mu)= \biggr \{f: [0,\infty) \rightarrow \mathbb{R}^2 \biggr | f \text{ is continuous, } \int_0^{\infty} \|f (x)\|_2^2\mu (\mathrm{d}x) < \infty, \biggr \}\nonumber
\end{eqnarray}
where $\|\cdot\|_2$ is the Euclidean norm in $\mathbb{R}^2$.
For our purposes, we use $\rho$ as the usual Euclidean metric on $\Fbb_2^2 (\mu) :=\Fbb_2 (\mu) \times \Fbb_2 (\mu)$.
Namely, for $(f_{11},f_{12}),(f_{21},f_{22}) \in  \Fbb_2^2 (\mu)$, we use
\begin{equation}
\label{eq:l_2 metric}
 \rho ( (f_{11},f_{12}) , (f_{21},f_{22}))^2
 := \|f_{11}-f_{21}\|_2^2
     +  \|f_{12}-f_{22}\|_2^2,
\end{equation}
where $\|f_{1i}-f_{2i}\|_2^2=\int_0^{\infty}\|f_{1i}(x) -f_{2i}(x)\|_2^2 \mu (\mathrm{d}x), i=1,2 $. Here and henceforth we assume that the trajectories $f_{11},f_{12},f_{21},f_{22}$ all span across the same length of time.
 Note that this choice of metric satisfies the criteria in Proposition~\ref{proposition:proper_metric}. Also, equivalently, to define similar rotation and translation invariant metrics on $\Fcal_{[s,t)}$, for any  $s<t\neq \infty$, we can simply choose any probability measure $\mu$ with support on $[s,t)$.

Proposition~\ref{proposition:computation} below provides a simple method to explicitly compute the metric $d$ between interactions, when $\rho$ is given by~\eqref{eq:l_2 metric}.
We will need the following notation:

\begin{enumerate}
    \item[ (A1)] For $(f_{11},f_{12}),(f_{21},f_{22}) \in \Fbb_2^2 (\mu)$, let $UDV^T$ be the singular value decomposition for the matrix given by 
    $$\sum_{i=1}^2 \int_0^\infty 
    \left (f_{2i} (x)-
        \bar{f_{2\cdot}}(x)
    \right) 
    \left(f_{1i} (x)-
        \bar{f_{1\cdot}}(x)
    \right)^T \mu (\mathrm{d}x),$$
\end{enumerate}

\noindent where $\bar{f_{2\cdot}}(x) = \int_0^\infty (f_{21} (x) + f_{22} (x))/2 \ \mu (\mathrm{d}x)$ and \\$\bar{f_{1\cdot}}(x) = \int_0^\infty (f_{11} (x) + f_{12} (x))/2 \ \mu (\mathrm{d}x)$.

Each of the summands in (A1) form a $2 \times 2$ dimensional matrix. Here, $\bar{f_1}(x)$ denotes the elementwise integration of the $2 \times 1 $
vector $(f_{11} (x) + f_{12} (x))/2$. Moreover, the outer-integral in each of the summands in (A1) is an elementwise integral of the $2 \times 2$ matrix integrand formed by matrix multiplication of the $ 2 \times 1$ vector $\left (f_{21} (x)-\bar{f_{2\cdot}}(x)\right)$ and the $1 \times 2$ vector $\left(f_{11} (x)-\bar{f_{1\cdot}}(x)\right)^T$.

\begin{proposition}
\label{proposition:computation}
Assume $f_{11},f_{12},f_{21},f_{22} \in \Fbb_2 (\mu)$. Let $UDV^T$ be the singular value decomposition as in  (A1). Then, 
\begin{eqnarray}
\label{eq:computation}
\inf_{ (f'_1,f'_2) \in  (\Ocal,\Ccal)_{(f_{21},f_{22})}} (\rho (  (f_{11},f_{12}), (f'_1,f'_2)))^2  = - 2 \operatorname{trace}\left (D 
 \begin{bmatrix}
    1 & 0 \\
    0 & \operatorname{det} (V^TU)
  \end{bmatrix}\right)  \nonumber \\ + \sum_{i=1}^2\int_0^\infty \biggr\|f_{2i} (x) -
        \bar{f_{2\cdot}}(x)
    \biggr\|_2^2  \nonumber  +  \biggr\|f_{1i} (x)-
        \bar{f_{1\cdot}}(x)
    \biggr\|_2^2\mu (\mathrm{d}x).
\end{eqnarray}
The optimal $\Ocal,\Ccal$ that define the infimum are given by :
\begin{align}
\begin{split}
   \tilde{\Ocal} &= V^T \begin{bmatrix}
    1 & 0 \\
    0 & \operatorname{det} (V^TU)
  \end{bmatrix}U,
  \\
   \tilde{\Ccal} &=
        \bar{f_{1\cdot}}(x)
    -\tilde{\Ocal} \cdot\left (
        \bar{f_{2\cdot}}(x)
    \right).
\end{split}
\end{align}
\end{proposition}

The proof of the proposition is discussed in Section~\ref{proof:proposition::computation}.
The problem discussed in Propostion~\ref{proposition:computation} is a version of the well-known  \textbf{least root mean square deviation} problem. It was first solved by the Kabsch algorithm~\cite{Kabsch-76,Kabsch-78}. A more computationally efficient method to compute the optimal $\Ocal, \Ccal$ was later obtained using the theory of quarternions~\cite{Quarternion-86,Quarternion-04}.

\section{Quantifying distributions of primitives}
\label{Section: Distribution_primitives}
The metric defined above can be used to obtain clusters of interactions, in addition to evaluating the overall quality and stability of a particular clustering method. Our starting point is to note that the problem of clustering or summarizing interactions can be formalized as finding a discrete distribution on the space of interactions.
More specifically, one needs to obtain a discrete probability distribution on interactions, where each supporting atom represents a typical interaction (namely, an interaction \emph{primitive}) and the mass associated with each atom represents the proportion of a cluster.
From this perspective, an objective that naturally arises is to minimize a distance from the empirical distribution of interactions to a discrete probability measures with a fixed number, say $k$, of supporting atoms, which represent the primitives. 
An useful tool for defining distance metrics on the space of distributions
arises from the theory of optimal transport \cite{Villani-03}.

Optimal transport distances enable comparisons of distributions in arbitrary structured and metric spaces by accounting for the underlying metric structure.
They have been increasingly adopted to address clustering in a number of contexts~\cite{Pollard-82-kmeans,Graf-Luschgy-00,ho2017multilevel}.
For instance, it is well-known that the problem of determining an optimal finite discrete probability measure minimizing the second-order Wasserstein distance $W_2$ to the empirical distribution of the data is directly connected to the k-means clustering problem (discussed in Section III in details). Inspired by this connection, we will seek to summarize the distribution of interactions appropriately. To this end, we will define Wasserstein distances for distributions of interactions as follows, by accounting for the metric structure developed in the previous section.

Let $d$ be a distance metric on $\Fbbsim$, where $\sim$ is the equivalence relation defined in Eq.~\eqref{eq:metric_quotient_space}. 

Fix $ [(f_{11},f_{12})] \in \Fbbsim$. Here, $[(f_{11},f_{12})]$ denotes the equivalence class corresponding to interaction $(f_{11},f_{12})$ relative to the equivalence relation $\sim$ and $ \Fbbsim$ denotes the collection of all such classes of interactions. Let $P (\Fbbsim)$ denote all probability measures on $\Fbbsim$. For a fixed order $r\geq 1$, define the following subset of $P(\Fbbsim)$ subject to a moment-type condition using the metric $d$:

\begin{eqnarray}
& &\Pcal_r (\Fbbsim) := \biggr\{ G \in P (\Fbbsim) | \nonumber \\ & &  \int d^r ([ (f_{21},f_{22})],[ (f_{11},f_{12})]) \mathrm{d}G ([ (f_{21},f_{22})]) < \infty \biggr\}. \nonumber
\end{eqnarray}
 This class of probability measures can be shown to be independent of the choice of $[  (f_{11},f_{12})]$ and therefore the collection of order-$r$ integrable probability measures on the quotient space $\Fbbsim$ is independent of the choice of the base class $[ (f_{11},f_{12})]$. We arrive at the following distance metric to compare between probability measures on the quotient space $\Fbbsim$. This is an instantiation of Wasserstein distances that arise in the theory of optimal transport in metric spaces~\cite{Villani-03}.

\begin{definition}[\textbf{Wasserstein distances}]
Let $F,G \in \Pcal_r (\Fbbsim)$. The Wasserstein distance of order $r$ between $F$ and $G$ is defined as:
\begin{eqnarray}
W_r (F,G) := \biggr( \underset{\pi \in \Pi (F,G)}{\inf} \int d^r ([ (f_{11},f_{12})],[ (f_{21},f_{22})]) 
\mathrm{d}\pi ([ (f_{11},f_{12})],[ (f_{21},f_{22})])\biggr)^{1/r} \nonumber,
\end{eqnarray}
where $\Pi (F,G)$ is the collection of all joint distributions on $\Fbbsim \times \Fbbsim$ with marginals $F$ and $G$.
\end{definition}

\subsection{Wasserstein barycenter and k-means clustering}
\label{ssection:k-means}
In this section, we present the Wasserstein barycenter problem and highlight its connection to the k-means formulation. 

\paragraph{Wasserstein barycenter problem} Fixing the order $r=2$, let $P_1,P_2,\ldots,P_N \\ \in \Pcal_2 (\Fbbsim)$ be probability measures on $\Fbbsim$. Their second-order Wasserstein barycenter is a probability measure $\bar{P}_{N,\lambda} $ such that 

\begin{eqnarray}
\bar{P}_{N,\lambda}=\underset{P \in P_2 (\Fbbsim) }{\text{argmin}} \sum_{i=1}^N \lambda_i W_2^2 (P,P_i). \nonumber
\end{eqnarray}
The Wasserstein barycenter problem was first studied by~\cite{Agueh-Carlier-11}.
When $P_i$ are themselves finite discrete probability measures on arbitrary metric spaces, efficient algorithms are available for obtaining locally optimal solutions to the above~\cite{Cuturi-Barycenter-14}.

\paragraph{k-means clustering problem} The k-means clustering problem, when adapted to obtaining clusters in a non-Euclidean space of interactions, can be viewed as solving for the set $S$ of $k$ elements $[ (g_{11},g_{12})], \ldots, [ (g_{k1},g_{k2})] \in \Fbbsim$ such that, given samples $ (f_{11},f_{12}),\ldots, (f_{n1},f_{n2}) \in \Fcal^2 (\mu) $
\begin{eqnarray}
\label{eq:k_means}
S= \underset{T: |T| \leq k}{\text{argmin}} \sum_{i=1}^n \inf_{[ (f'_1,f'_2)] \in T }d^2 ([ (f_{i1},f_{i2})],[ (f'_1,f'_2)]).
\end{eqnarray}

It can be shown that this is equivalent to finding a discrete measure $P$ which solves the following for the choice $r=2$: 
\begin{eqnarray}
\label{eq:wasserstein-k_means}
\inf _{P \in \Ocal_k (\Fbbsim)} W_r (P,P_n),
\end{eqnarray}
where $P_n$ is the empirical measure on $\Fbbsim$, i.e., $P_n$ places mass $1/n$ on equivalence class sample $[(f_{i1},f_{i2})]$ for all $i =1,\ldots, n$, and $\Ocal_k(\Fbbsim)$ is the set of all measures in $\Fbbsim$ with at most $k$ support points. (It is interesting to note that Eq.~\eqref{eq:wasserstein-k_means} is a special case of the Wasserstein barycenter problem for $N=1$ and $r=2$.)

At the high level our approach is simple: we seek to summarize the empirical data distribution of interactions using a k-means-like approach, but there are several challenges due to the complex metric structure exhibited by the non-Euclidean space of interactions. Finding the exact solution even in the simplest cases is an NP-hard problem. The most common method to approximate the solution is the use of iterative steps similar to Lloyd's algorithm~\cite{Lloyd-algo} for solving the Euclidean k-means problem. However, the computation of cluster centroids at each iteration of Lloyd's algorithm when applied to the non-Euclidean metric $d$ is non-trivial. Moreover, the computation of pairwise distances between equivalence classes of interactions is non-trivial. In the next subsection we present some approximate solutions to Eq.~\eqref{eq:k_means}.

\subsection{Approximations for non-Euclidean $k$-means clustering}
\label{ssection:approximation}
The primary objective for this section is to obtain a robust representation for the distribution over interaction primitives. Although the empirical distribution of interactions provides an estimate of the distribution over primitives, it suffers from lack of robustness guarantees. A robust $k$-approximation for the empirical distribution is formalized by Eq.~\eqref{eq:wasserstein-k_means}. For order $r=2$ this is equivalent to solving the k-means problem given by Eq.~\eqref{eq:k_means} for the interaction scenarios. The computational problem for computing exact centroids of k-means clusters is cumbersome and generally not solvable for arbitrary distance metrics $d$. To overcome such  challenges we propose three separate methods to obtain approximate solutions to Eq.~\eqref{eq:k_means}. The first approach is a standard application of multi-dimensional scaling technique. The second and third approaches are based on other geometric ideas to be described in the sequel. 
\subsubsection{Multidimensional Scaling}
\label{sssection:mds}
Multi-dimensional scaling (MDS) provides a way to obtain a lower dimensional representation of high-dimensional and/or non-Euclidean space elements while approximately preserving some distance measure among data points. Given a distance (a.k.a. dissimilarity) matrix 
$D= (d_{ij})_{1\leq i,j\leq n}$, which collects all pairwise distance among the $n$ data points using a notion of distance such as metric $d$ described earlier, MDS finds points $x_1,\ldots,x_n \in \mathbb{R}^m$, for some small dimension $m$, such that 
\begin{eqnarray}
\label{eq:MDS}
\{x_1,\ldots, x_n\} = \underset{y_1,\ldots,y_n \in \mathbb{R}^m}{\text{argmin}} \sum_{i,j=1}^n  (\|y_i-y_j\|-d_{ij})^2
\end{eqnarray}

In order to apply the k-means clustering technique to our MDS representation, the following implicit assumption is required:

\begin{itemize}
    \item[ (C1)] Each of the cluster centroids for the k-means problem corresponds to an interaction in the data sample.
\end{itemize}

Given (C1), Eq.~\eqref{eq:k_means} can be reformulated as follows.
\paragraph{Approximate k-means} Given interaction samples $ (f_{11},f_{12}),\ldots, (f_{n1},f_{n2}) \in \Fcal^2(\mu)$, find a set $S \subset \{1,\ldots,n\}$ such that,
\begin{eqnarray}
\label{eq:k_means_approx}
S= \underset{T: |T| \leq k}{\text{argmin}} \sum_{i=1}^n \min_{j \in T }d^2 ([ (f_{i1},f_{i2})],[ (f_{j1},f_{j2})]).
\end{eqnarray}
The approximate k-means problem in Eq. ~\eqref{eq:k_means_approx} differs from the k-means problem ~\eqref{eq:k_means} in that instead of finding primitives that are the global minimizer (and hence correspond to the cluster means), we look for the primitive that is closest to all other interactions in its cluster. The advantage of this approach is that we do not need explicitly the inverse map that goes from the MDS representation back to the interaction space. We summarize this approach as Algorithm~\ref{algo: primitive clustering} in the following.

\begin{algorithm}[ht]
\caption{Clustering interactions}
\label{algo: primitive clustering}
\algsetup{linenosize=\small}
\scriptsize
Input: interaction sample $\{(f_{i1},f_{i2})\}_{i=1}^{n}$\\
Output: $k$ interaction primitives
\begin{algorithmic}[1]
\STATE Obtain $x_1,\ldots,x_n$ as solution of MDS Eq.~\eqref{eq:MDS} with $d_{ij}=d ([ (f_{i1},f_{i2})],[ (f_{j1},f_{j2})])$.
\STATE  Perform k-means on $x_1,\ldots,x_n$ to obtain the centroids.
\STATE Approximate the centroids with points $x_i \in \mathbb{R}^m$ which are closest in $\|\cdot\|$ distance to the centroids, $\Gamma_1, \Gamma_2, \ldots, \Gamma_k$.
\STATE Return as primitives the $k$ interaction sample corresponding to these approximate centroids, $\{(g_{j1},g_{j2})\}_{j=1}^{k}$.
\end{algorithmic}
\end{algorithm}

\subsubsection{Geometric Approximations}
\label{sssection:geometric_approx}

A major computational challenge to solving Eq.~\eqref{eq:k_means_approx} lies in the SVD decomposition of the  Procrustes distances (Eq.~\eqref{eq:rotation_translation_invariance}) relative to each pair of interactions. There require $O(n^2)$ such decomposition.  To avoid this, we instead consider a geometric approximation of the Procrustes distance, inspired by work from the field of morphometrics \cite{stats_shape_analysis}. 

Consider two interactions $(f_{i1},f_{i2})$ and $(f_{j1},f_{j2})$. Then, by an application of triangle inequality,
\begin{eqnarray}
\label{eq:geo_approx_1}
& & \inf_{ (f'_1,f'_2) \in  (\Ocal,\Ccal)_{ (f_{j1},f_{j2})}}\rho ( (f_{i1},f_{i2}), (f'_1,f'_2)) \\ 
 &=&\inf_{O_1 \in SO(2), c_1 \in \mathbb{R}^2}\rho ( (f_{i1},f_{i2}), O_1\odot(f_{j1},f_{j2})+ c_1) \nonumber \\
 & \leq & \inf_{O_1 \in SO(2), c_1 \in \mathbb{R}^2}\rho ( (f_{11},f_{12}), O_1\odot(f_{j1},f_{j2})+ c_1) \nonumber \\& + &\inf_{O_2 \in SO(2), c_2 \in \mathbb{R}^2}\rho ( (f_{11},f_{12}), O_2\odot(f_{i1},f_{i2})+ c_2).\nonumber
\end{eqnarray}
 Eq.~\eqref{eq:geo_approx_1} shows that knowledge of  optimal rotational matrices and translation vectors for computing the distances  $d ([ (f_{i1},f_{i2})],[ (f_{11},f_{12})])$ and \\ $d ([ (f_{j1},f_{j2})],[ (f_{11},f_{12})])$ can provide an upper bound for computing the distance between the $i^{th}$ and $j^{th}$ pair of interactions. Therefore, we can provide a reasonable upper bound for all the $n^2$ pairwise distances by simply performing only $O(n)$ SVD decompositions. This approach, which we call the $\textit{first geometric approximation}$, is summarized in Algorithm~\ref{algo:first_geom_approx}.
 
\begin{algorithm}[!ht]
\caption{First Geometric Approximation}
\label{algo:first_geom_approx}
\algsetup{linenosize=\small}
\scriptsize
Input: $\{(f_{i1},f_{i2})\}_{i=1}^{n}$\\ 
Output: $k$ centroids
\begin{algorithmic}[1]
\FOR{$i = 1, 2, \dots, n$}
\STATE Center and reorient $(f_{i1},f_{i2})$ to $(f_{11},f_{12})$ using Algorithm \ref{algo: primitive reorienting}. 
\ENDFOR 
\STATE Perform k-means on the centered and oriented $\{(f_{i1},f_{i2})\}_{i=1}^{n}$ to obtain the centroids, $\{(\Gamma_{j1}, \Gamma_{j2})\}_{j = 1}^k$. 
\STATE Return the centroids, $\{(\Gamma_{j1}, \Gamma_{j2})\}_{j = 1}^k$.
\end{algorithmic}
\end{algorithm}

\begin{algorithm}[!ht]
\caption{Second Geometric Approximation}
\label{algo:second_geom_approx}
\algsetup{linenosize=\small}
\scriptsize
Input: $\{(f_{i1}, f_{i2})\}_{i=1}^{n}$\\ 
Output: $k$ centroids
\begin{algorithmic}[1]
\STATE Randomly assign interaction samples $\{(f_{i1},f_{i2})\}_{i=1}^{n}$ to $k$ clusters. Let $z_i$ indicate the cluster assignment.
\WHILE{k-means convergence criterion has not been met}
\FOR{$k' = 1, 2, \dots, k$}
\STATE Center and orient all interaction samples $(f_{i1},f_{i2})$ to $(f_{i_{k'}1},f_{i_{k'}2})$ using Algorithm \ref{algo: primitive reorienting} if $(f_{i_{k'}1},f_{i_{k'}2})$ is the first interaction sample such that $z_i = k'$ for $i = 1, 2, \dots n$. Denote these oriented and centered samples as $(f'_{i1},f'_{i2})(t)$.
\STATE Compute the centroid for cluster $j$, $(\Gamma_{j1},\Gamma_{j2})$, such that for  $t = 1, 2, \dots, t_m$,
\[
(\Gamma_{j1},\Gamma_{j2})(t) = \frac{1}{\countV{z_i = k}} \sum_{i:z_i = k} (f'_{i1},f'_{i2})(t)
\]
\ENDFOR
\FOR{$i = 1, 2, \dots, N$}
\FOR{$j = 1, 2, \dots, k$}
\STATE Center and orient $(f_{i1},f_{i2})$ to $(\Gamma_{j1}, \Gamma_{j2})$.
\STATE Compute the $L_2$ distance between the centered and oriented $(f_{i1},f_{i2})$ and $(\Gamma_{j1}, \Gamma_{j2})$.
\ENDFOR
\STATE Set $z_i = j$ if the smallest computed distance is from the centroid of cluster $j$. 
\ENDFOR
\ENDWHILE
\STATE Return the centroids, $\{(\Gamma_{j1}, \Gamma_{j2})\}_{j = 1}^k$.
\end{algorithmic}
\end{algorithm}

However, this gain in computation efficiency is also accompanied by a loss of statistical efficiency. To mitigate this tension between computational and statistical efficiency we propose a $\textit{second geometric approximation}$ which performs the approximation of Algorithm~\ref{algo:first_geom_approx} in batch form, where the batches comprise of the respective clusters. This procedure is described in Algorithm \ref{algo:second_geom_approx}.

\section{Experimental Results}
\label{Section: Results}

In this section we provide a demonstration of our methods for unsupervised learning of vehicle interactions.  In particular, we will evaluate the quality and stability of clustered primitives extracted from vehicle-to-vehicle interactions based on real-world experiments conducted in Ann Arbor, Michigan. In the literature for this application domain, a real-time interaction between two vehicles is also alternatively referred to as an encounter. In practice, the interactions between vehicles are represented by multi-dimensional time series of varying duration, which need to be further segmented into shorter time duration via suitable data processing techniques. 
\begin{figure*}[!tp]
\centering
\begin{subfigure}{.19\textwidth}
\centering
\includegraphics[width = 1\textwidth]{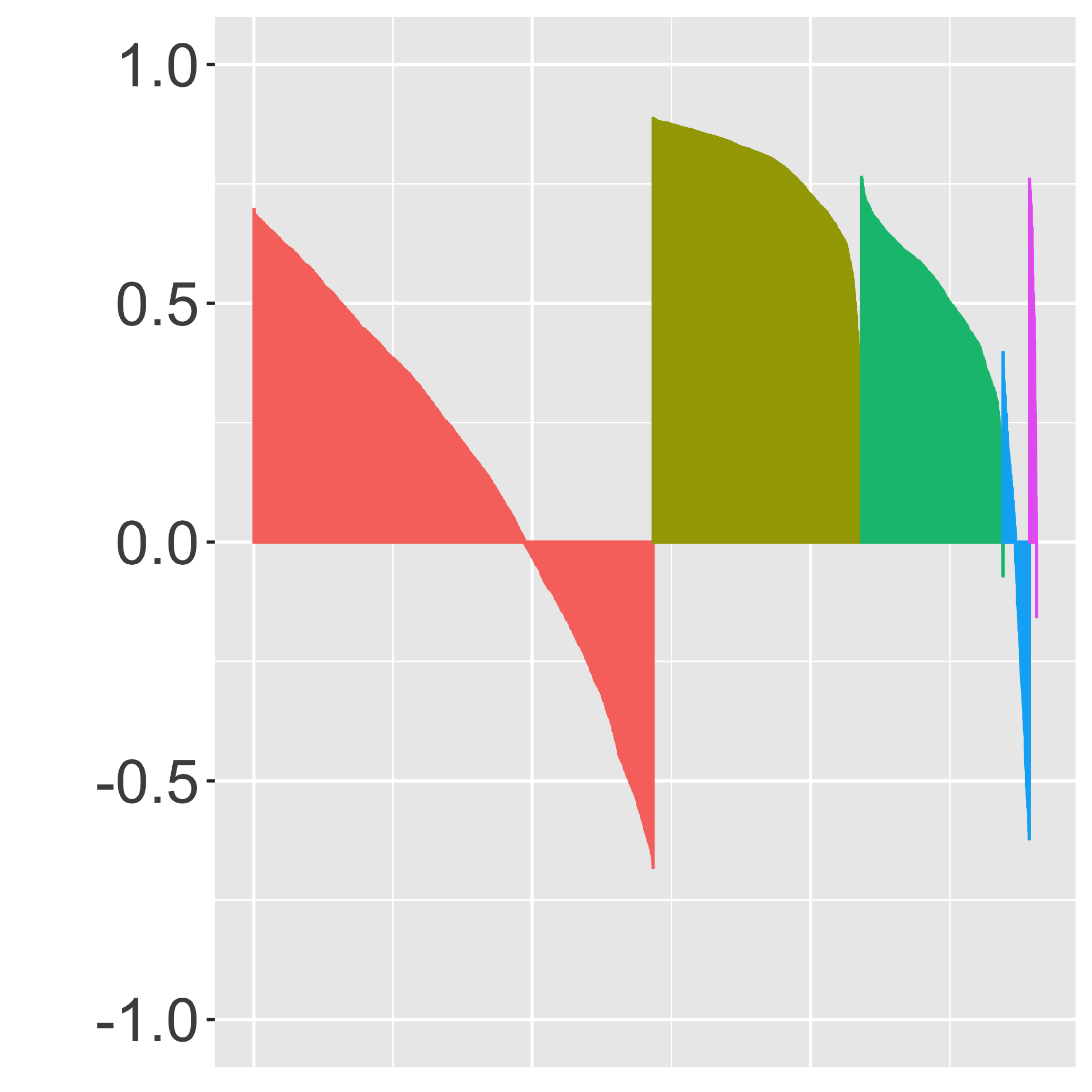}
\caption{Multidim. Scaling (cf. Section~\ref{sssection:mds})}
\end{subfigure}
\begin{subfigure}{.19\textwidth}
\centering
\includegraphics[width = 1\textwidth]{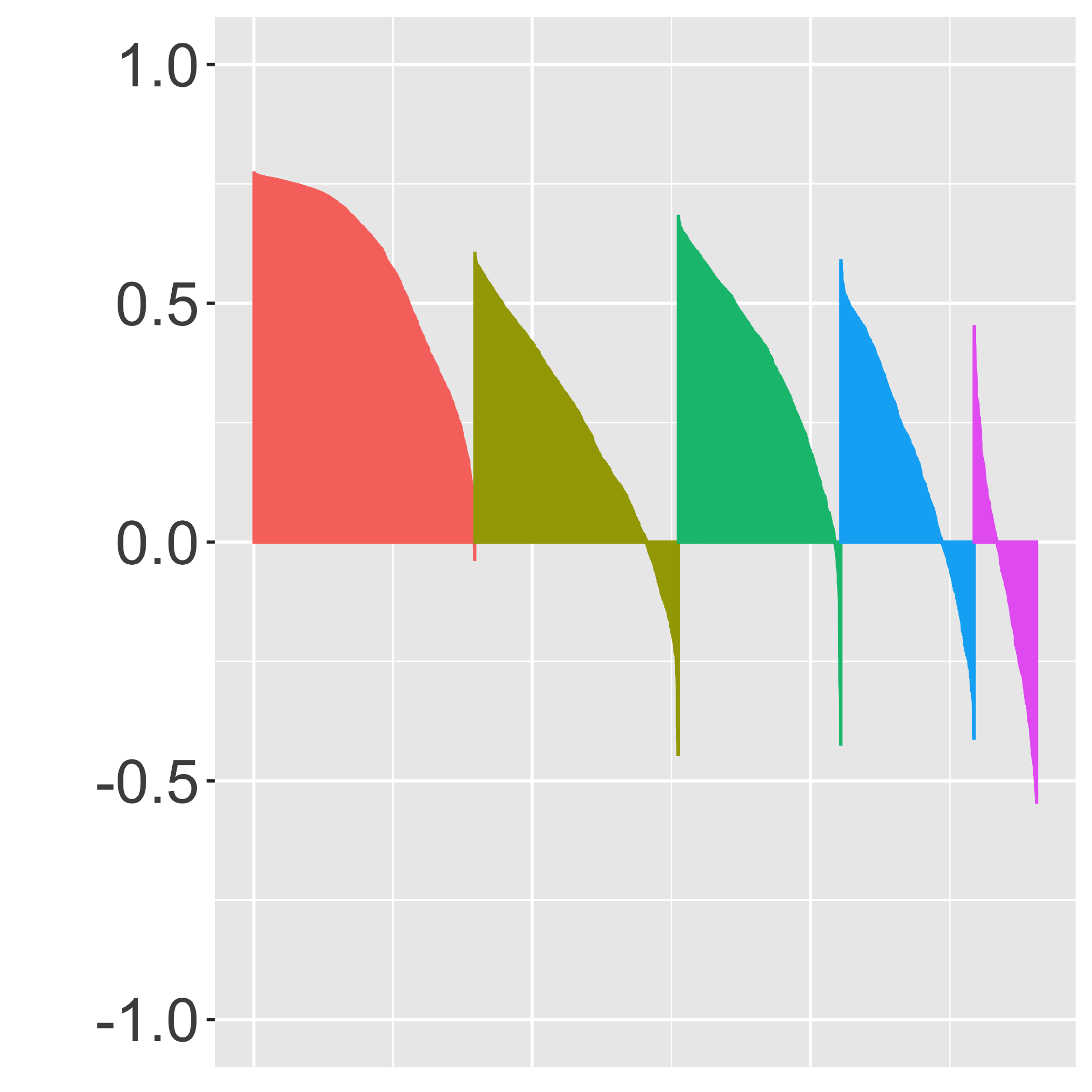}
\caption{First geometric approx. (cf. Section~\ref{sssection:geometric_approx})}
\end{subfigure}
\begin{subfigure}{.19\textwidth}
\centering
\includegraphics[width = 1\textwidth]{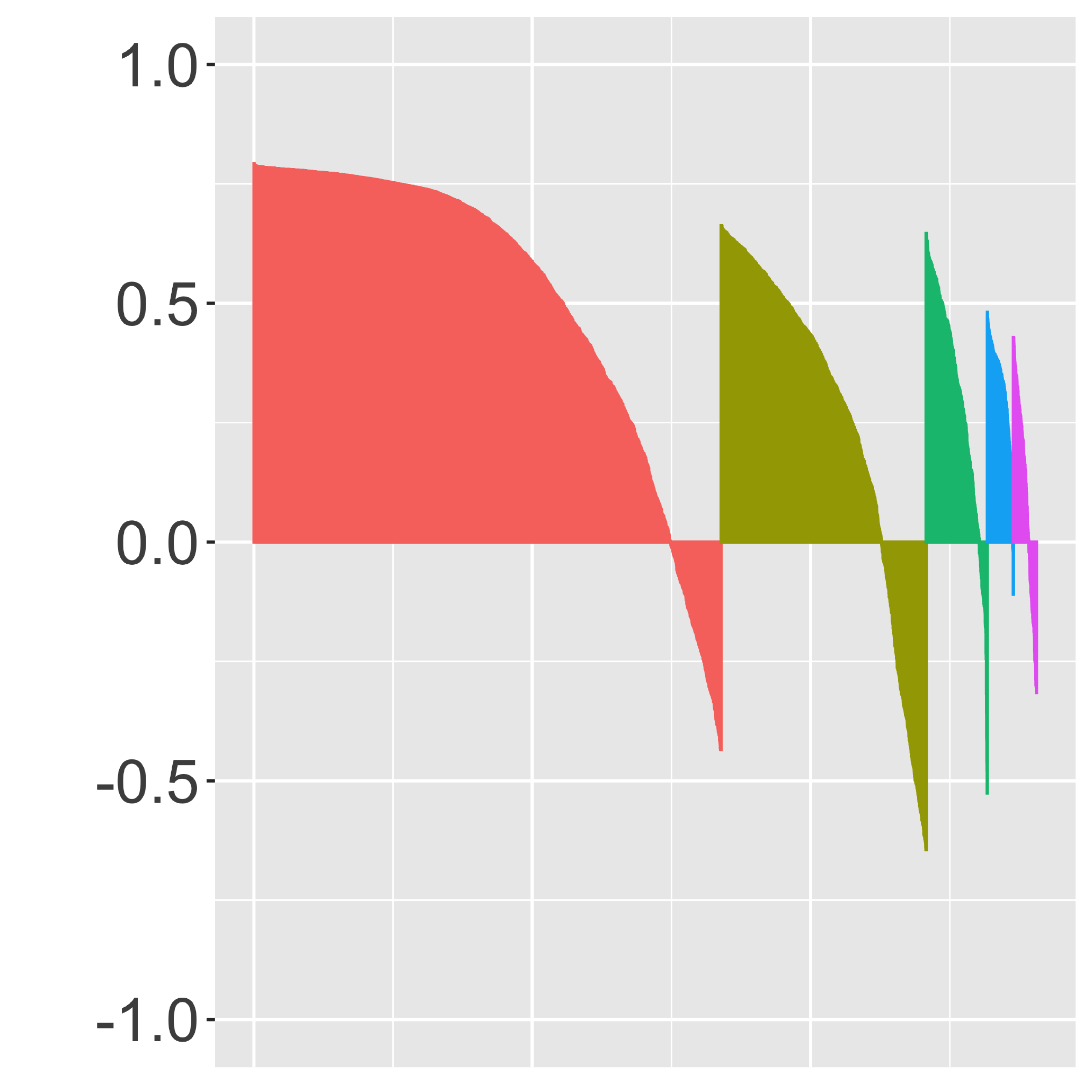}
\caption{Second geometric approx. (cf. Section~\ref{sssection:geometric_approx})}
\end{subfigure}
\begin{subfigure}{.19\textwidth}
\centering
\includegraphics[width = 1\textwidth]{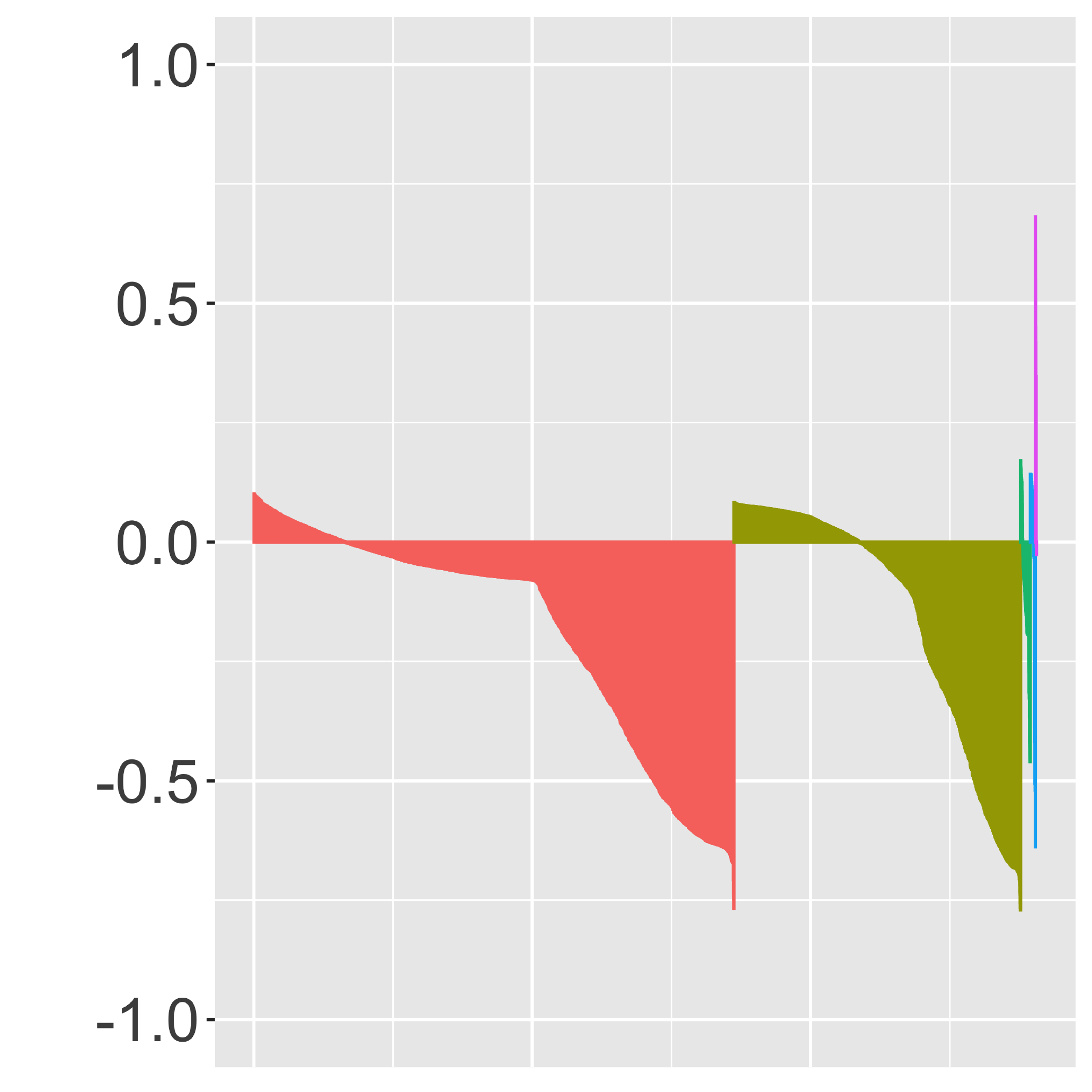}
\caption{Polynomial coefficients (cf. Section~\ref{ssection:clustering_quality})}
\end{subfigure}
\begin{subfigure}{.19\textwidth}
\centering
\includegraphics[width = 1\textwidth]{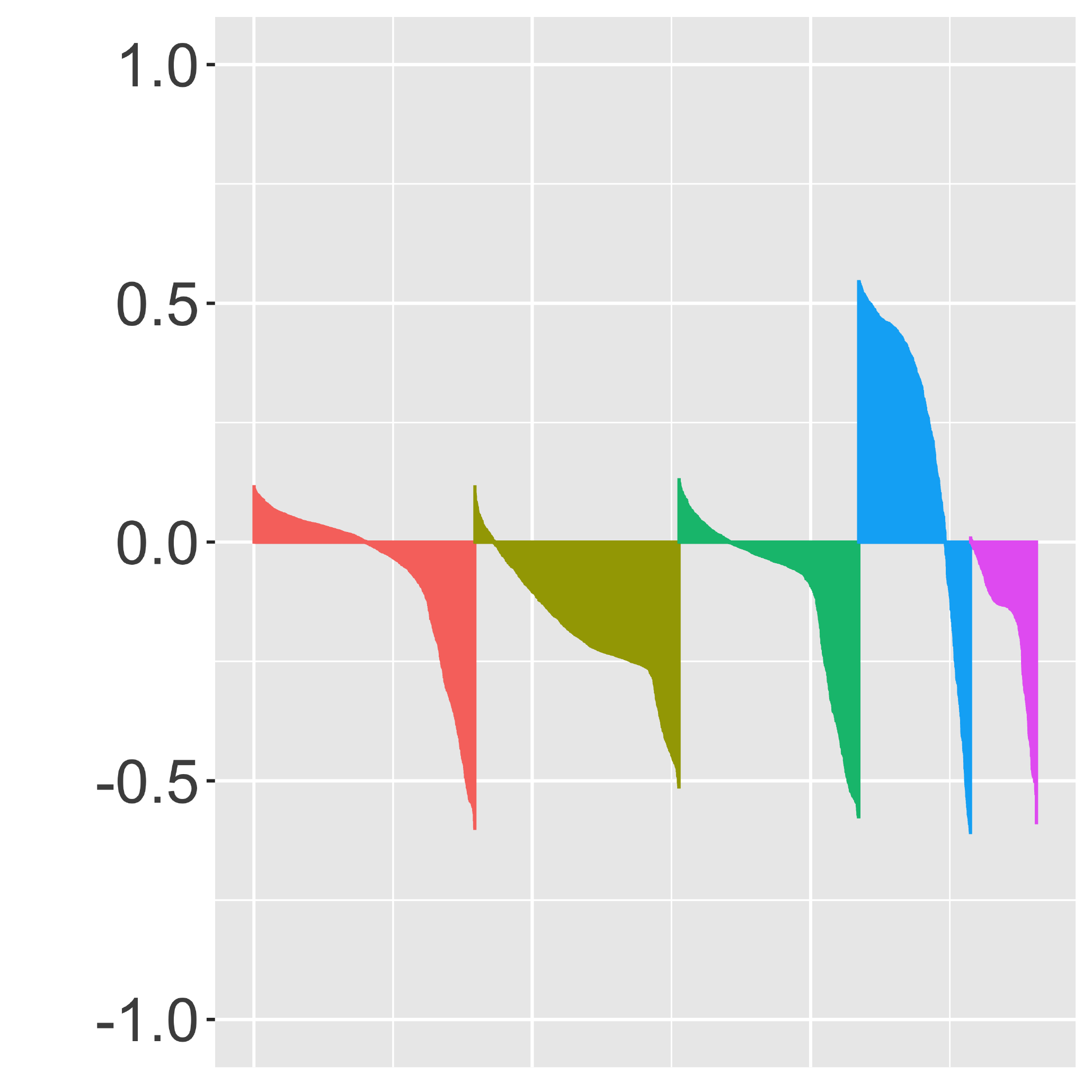}
\caption{DTW cost matrix (cf. \cite{Wang-Zhao-2017})}
\end{subfigure}\\
\caption{Silhouette plots for 5 clusters obtained under various approaches: 
}
\label{fig:two_step_silhouette_plots}

\centering
\resizebox{1\textwidth}{!}{%
\begin{tabular}{@{\extracolsep{5pt}} cccccc} 
\\[-1.8ex]\hline 
\hline \\[-1.8ex] 
 & \multirow{3}{*}{\shortstack[c]{\textbf{Total within}\\\textbf{Square Distance}}} & \multirow{3}{*}{\shortstack[c]{\textbf{Cluster 1}\\\textbf{Average within}\\\textbf{Square Distance}}} & \multirow{3}{*}{\shortstack[c]{\textbf{Cluster 5}\\\textbf{Average within}\\\textbf{Square Distance}}} & \multirow{3}{*}{\shortstack[c]{\textbf{Cluster 1}\\\textbf{Average between}\\\textbf{Square Distance}}} &  \multirow{3}{*}{\shortstack[c]{\textbf{Cluster 5}\\\textbf{Average between}\\\textbf{Square Distance}}}\\ 
 & & & & &\\
 & & & & &\\
\hline \\[-1.8ex] 
& & & & &\\
\textbf{MDS} & 10.68 & 1.74e-03 (1.53e-05) & 1.23e-02 (4.60e-04) & 1.55e-02 (4.57e-04) & 1.17e-01 (5.13e-03)\\
& & & & &\\
\textbf{First Geometric Approx.} & 13.32 & 9.50e-04 (2.96e-05) & 6.33e-03 (3.92e-04) & 5.01e-02 (4.97e-03)) & 1.12e-02 (2.92e-04)\\
& & & & &\\
\textbf{Second Geometric Approx.} & 274.27 & 3.97e-03 (1.42e-04) & 8.01e-02 (1.38e-03) & 1.86e-02 (8.43e-04) & 1.84e-02 (1.51e-03)\\
& & & & &\\
\textbf{Spline Coefficients} & 222.56 & 5.02e-02 (3.91e-03) & 9.40e-02 (1.05e-02) & 5.95e-02 (4.27e-03) & 2.05e-01 (2.33e-02)\\
& & & & &\\
\textbf{DTW Matrices} & 201.12 & 3.12e-02 (4.51e-03) & 9.00e-03 (3.41e-04) & 5.01e-02 (6.30e-03) & 9.38e-03 (3.79e-04)\\  
& & & & &\\
\hline \\[-1.8ex] 
\end{tabular} 
}
\captionof{table}{A table of the quantities from Eq. \eqref{eq:k_means_approx}, Eq. \eqref{eq:within_cluster_quality_stat_approx}, and Eq. \eqref{eq:between_cluster_quality_stat_approx}) for each method's cluster with the most interaction (Cluster 1) and cluster with the fewest (Cluster 5). Variance of these distances are included in parentheses. Note that the Procrustes distances were normalized so that the maximum distance between any interaction is 1.}
\label{table:two_step_cluster_quality_stats_approx} 

\end{figure*}

\subsection{Vehicle-to-vehicle (V2V) interaction data processing}

\label{Section:data}

We work with a real-world V2V interaction data set which is extracted from the naturalistic driving database generated by experiments conducted as part of the University of Michigan Safety Pilot Model Development (SPMD) program. In these experiments, dedicated short range communications (DSRC) technology was utilized for the communication between two vehicles. Approximately 3,500 equipped vehicles have collected data for more than 3 years. Latitude and longitude data of each vehicle was recorded by the by-wire speed sensor. The on-board sensor records data in 10Hz.

To investigate basic V2V interaction behaviors, a subset of 1400 driving scenarios was further filtered out from the SPMD's database. 
Each scenario consists of a time series of GPS locations and speeds of a pair of vehicles, which are mutually less than 100 metres apart. 
For our purposes, it is natural to posit that each scenario is inclusive of multiple shorter encounters through different time duration. Pre-processing of the data was therefore aimed at segmenting each scenario into more basic driving segments. These segments constitute basic building blocks from which we can meaningfully learn interaction primitives using a variety of clustering algorithms. 
The issue of segmentation is akin to identifying change points on functional curves embedded in a higher dimensional space.
We consider two different segmentation schemes for V2V interaction data processing.

The first segmentation scheme is detailed in Appendix \ref{Section: Spline}. It will be called a $\textit{two-step spline}$ approach, which goes as follows. Given an encounter, we fit it with cubic splines in two main steps. Here, the change points act as the knots. The first step involves identifying a large number of probable change points via a binary search approach to add change points if adding change points reduced the squared error between the fitted values and the observed data. The next step involves a single forward pass to remove excess change points from consideration in order to minimize the squared error with a penalty for the number of change points. We then segment each interaction at the knots. This segmentation technique created a set of 5622 basic V2V interactions to work with.

The second segmentation scheme is considerably more complex, as it is derived from a nonparametric Bayesian model for time series data, the sticky Hierarchical Dirichlet Process Hidden Markov Model (HDP-HMM)~\cite{Fox-etal-09}. This model extends the basic HMM by allowing the number of hidden Markov states to be unbounded, while encouraging the Markov process to be "sticky", that is, the state tends to be constant for a period of time (e.g., a car tends to go straight after a long period of time). 
For model selection, as we will elaborate later, one of the hyperparameters of sticky HDP-HMM is varied. Consequently, the number of basic V2V interactions varied from 8779 to 8829 with an average of 8799 interactions.

\begin{figure*}[!t]
\centering
\begin{minipage}{\textwidth}
\begin{subfigure}{.33\textwidth}
\centering
\includegraphics[width = 1\textwidth]{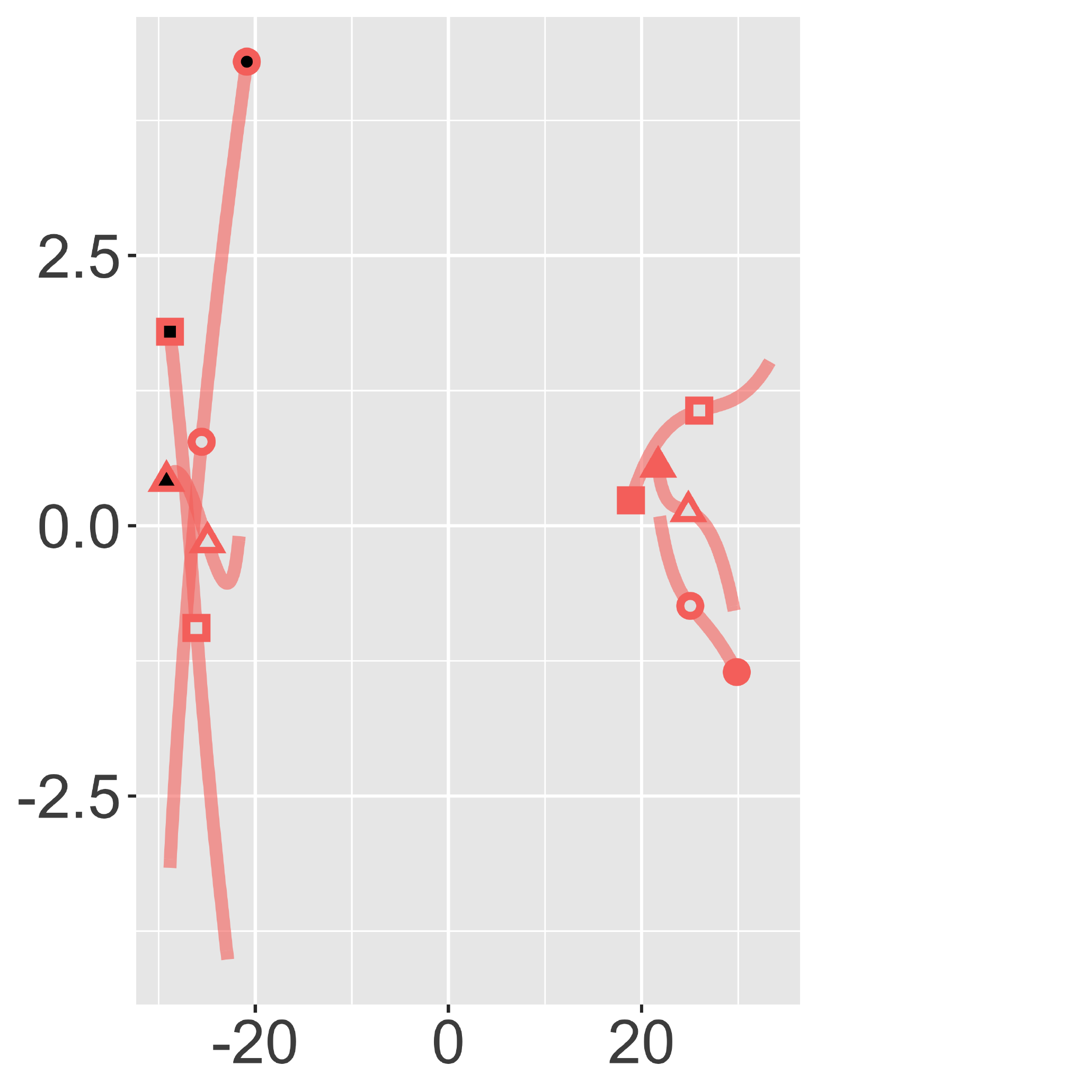}
\caption{Cluster 1}
\end{subfigure}
\begin{subfigure}{.33\textwidth}
\centering
\includegraphics[width = 1\textwidth]{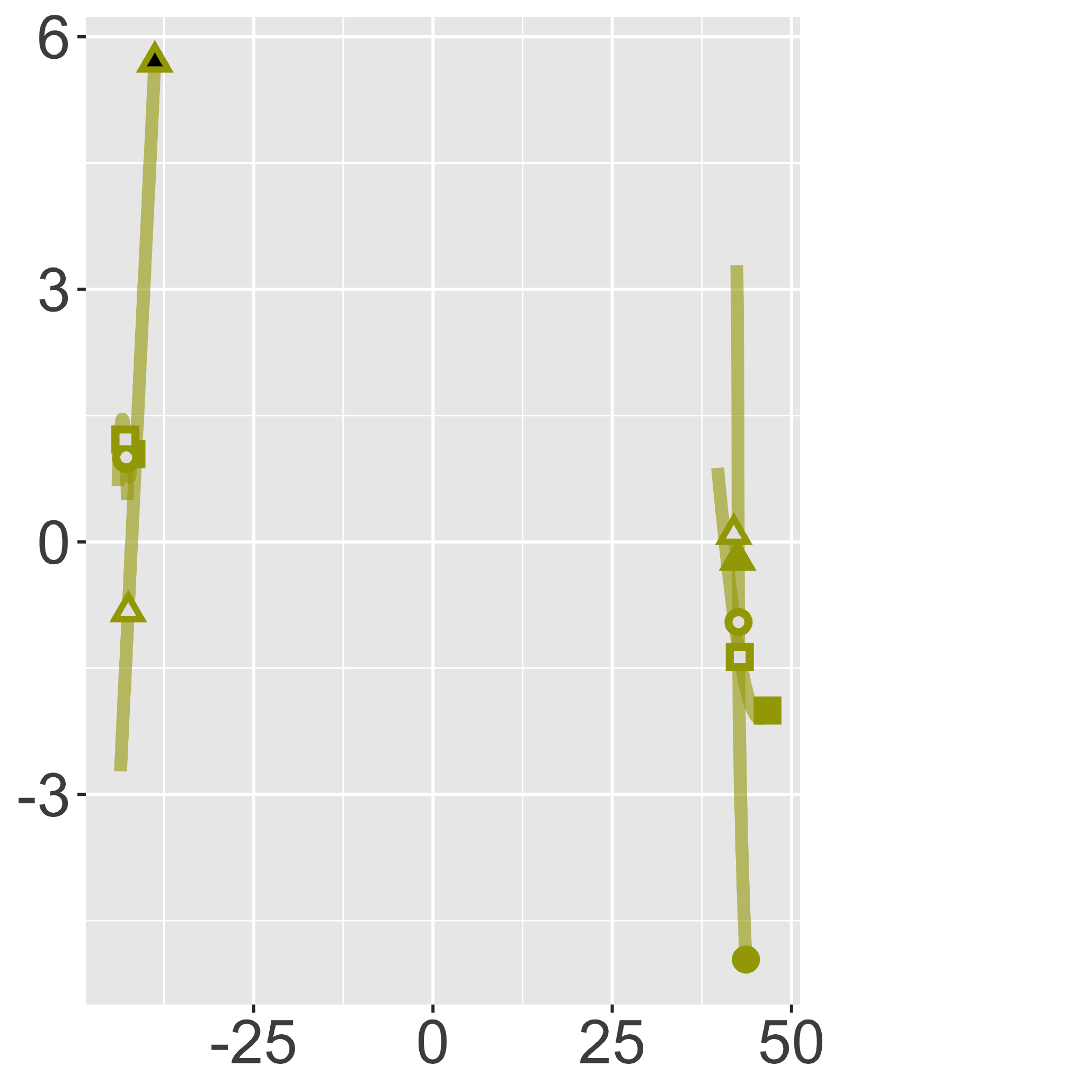}
\caption{Cluster 2}
\end{subfigure}
\begin{subfigure}{.33\textwidth}
\centering
\includegraphics[width = 1\textwidth]{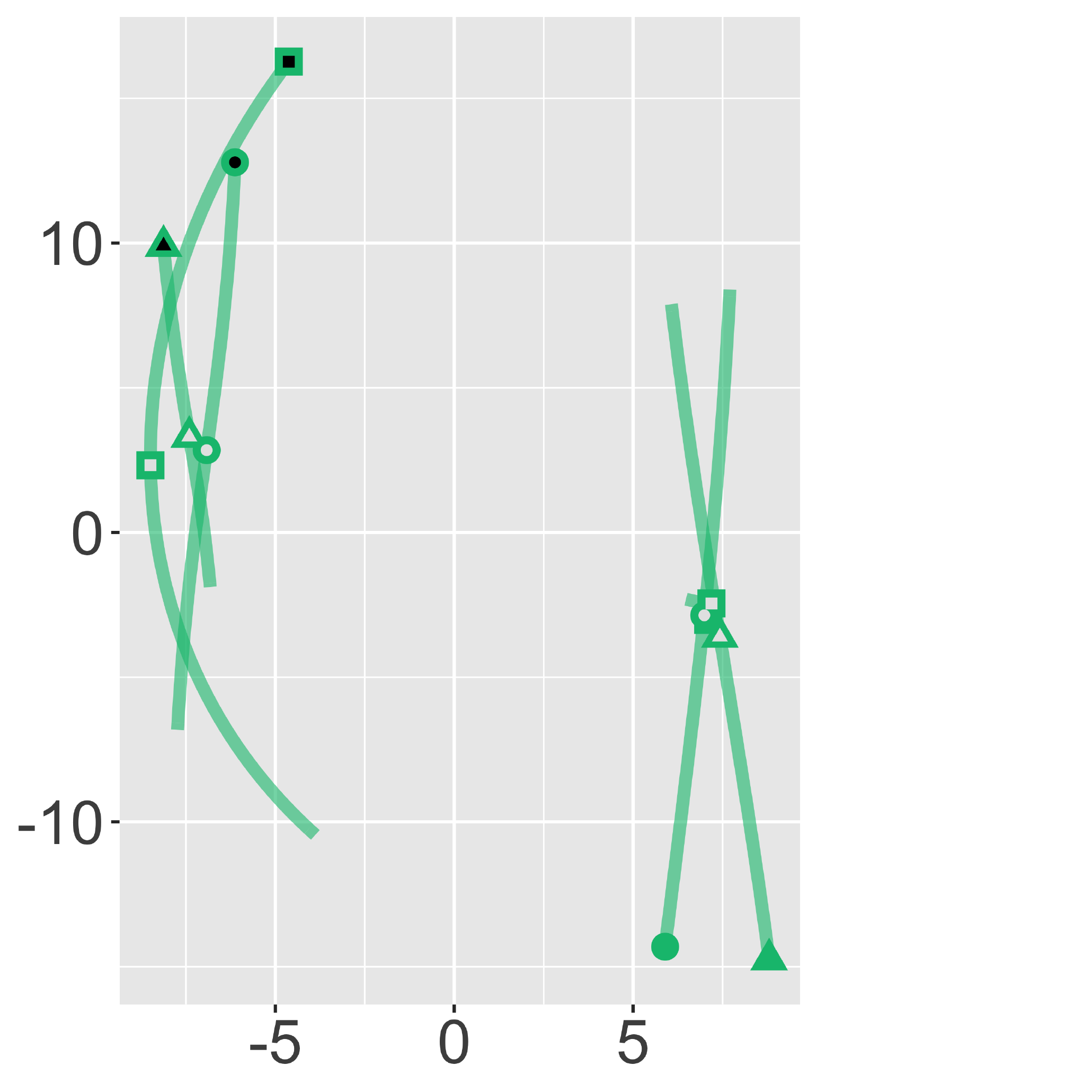}
\caption{Cluster 3}
\end{subfigure}
\end{minipage}

\begin{minipage}{\textwidth}
\begin{subfigure}{.33\textwidth}
\centering
\includegraphics[width = 1\textwidth]{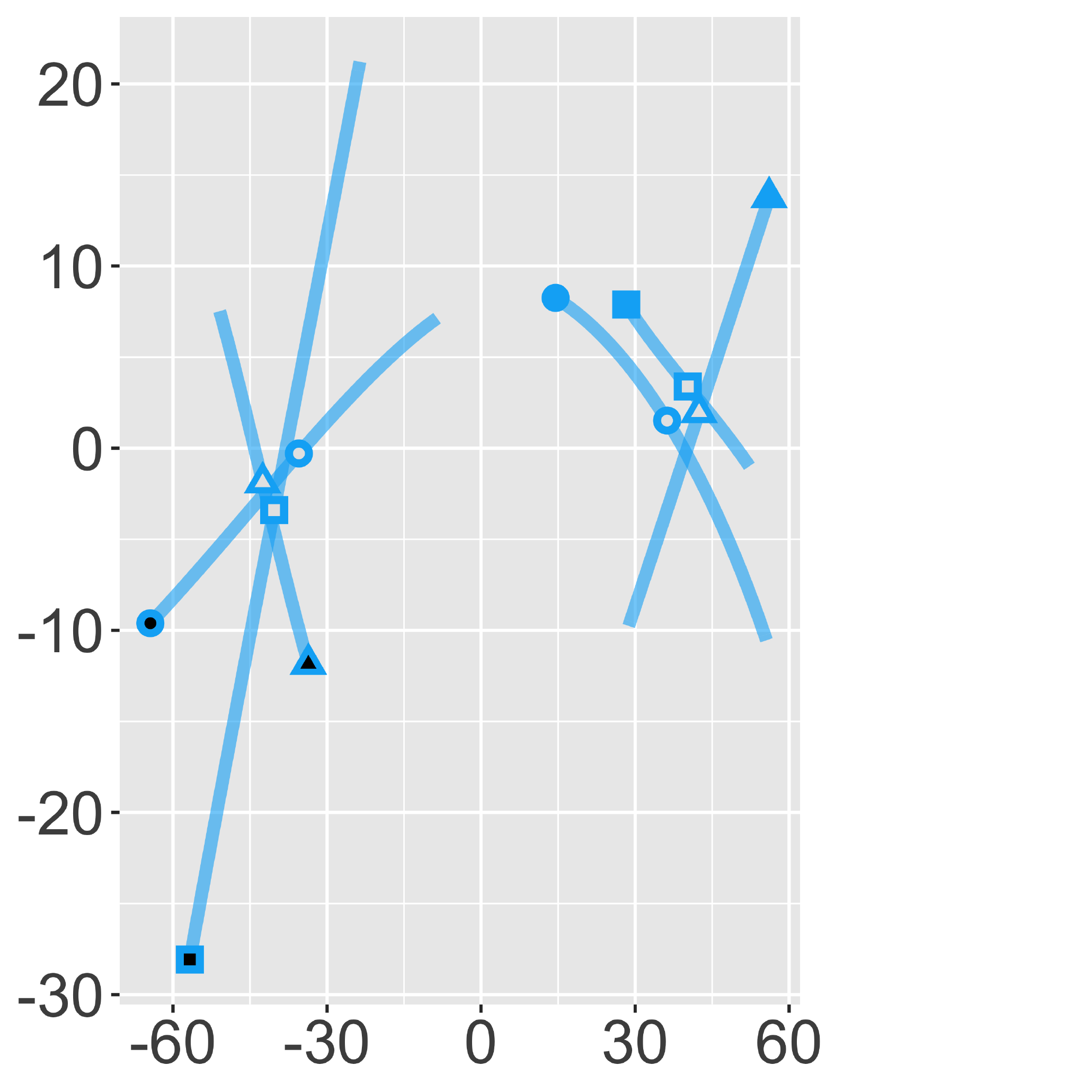}
\caption{Cluster 4}
\end{subfigure}
\begin{subfigure}{.33\textwidth}
\centering
\includegraphics[width = 1\textwidth]{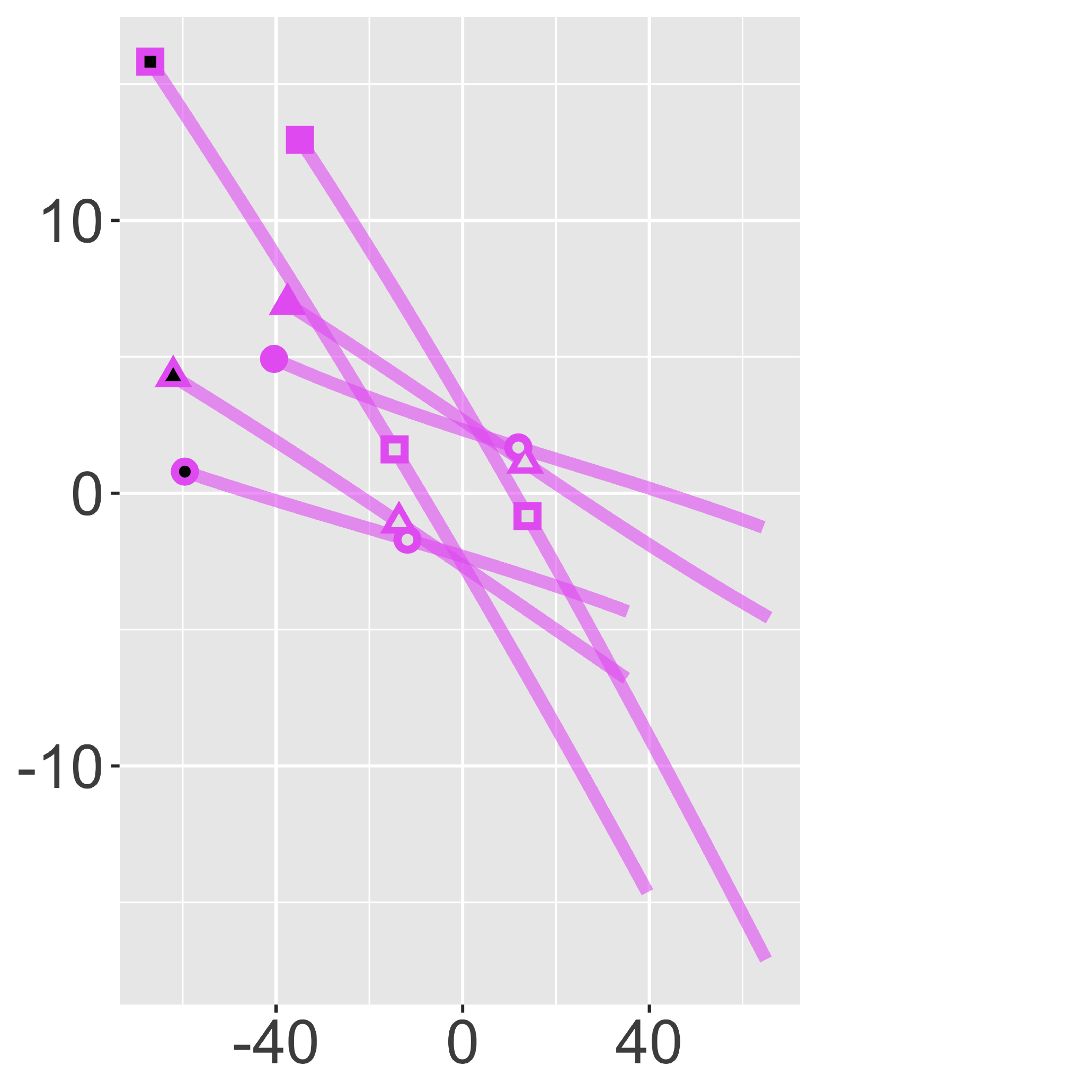}
\caption{Cluster 5}
\end{subfigure}
\begin{subfigure}{.33\textwidth}
\centering
\includegraphics[width = 1\textwidth]{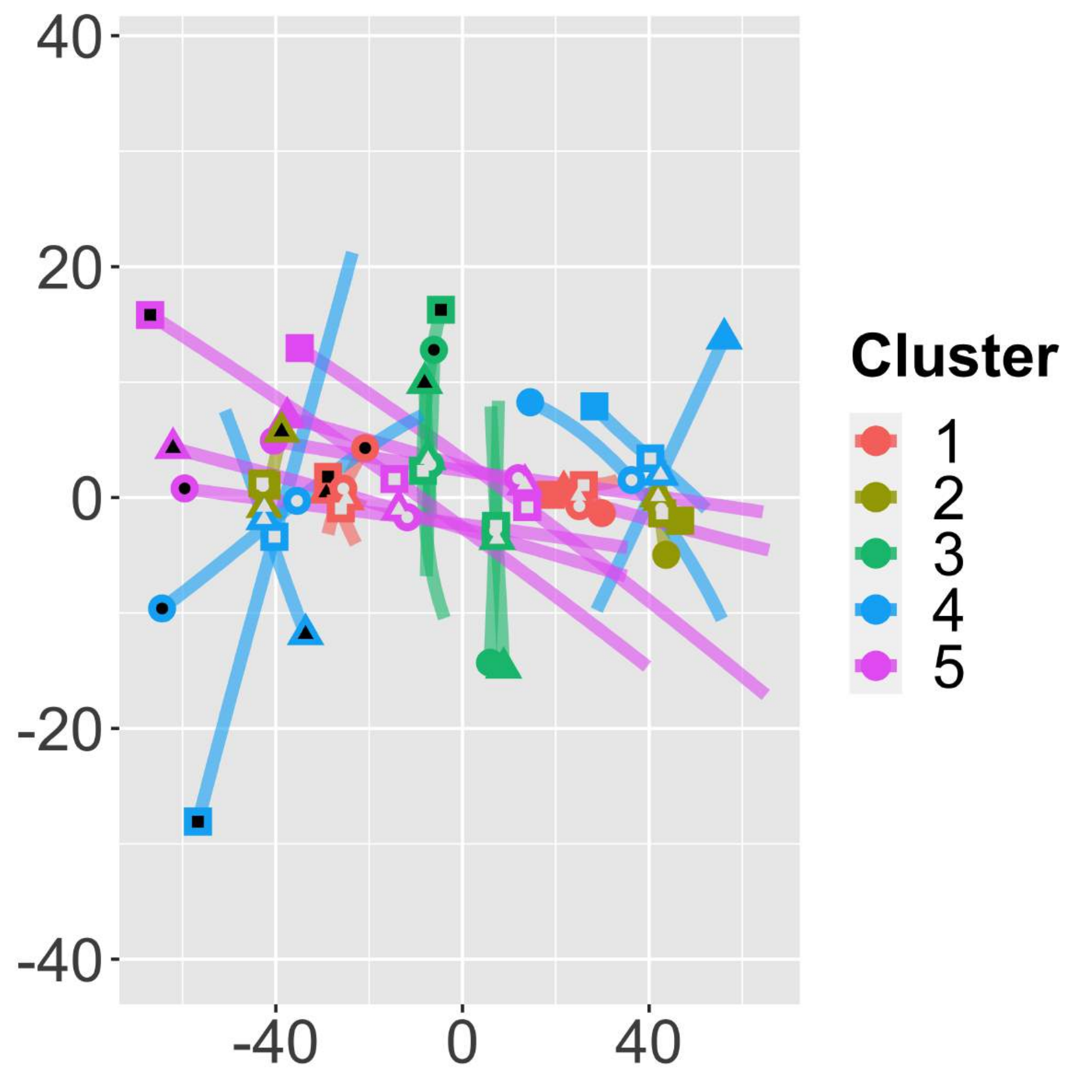}
\caption{All clusters}
\end{subfigure}
\end{minipage}
\caption{Plot of the three most typical interactions organized from the cluster with the most interaction to the cluster with the fewest for clustering using Multidimensional scaling (cf. Section \ref{ssection:approximation}). The interactions are centered and oriented using Algortihm \ref{algo: primitive reorienting} to $(t, 2t - 1, -t, 1 - 2t)$ for $t = 0, 0.01, \dots, 1$. The solid shapes and shapes with a black interior indicate the starting location of each interaction. The dot with the black interior indicates the second trajectory. Midpoints are indicated by dots filled in with a grey interior. Different shapes indicate different V2V interactions. Note that the individual cluster interactions plots are placed on their own scales.}
\label{fig:two_step_typical_encounter_plot_approx_mds}
\end{figure*}

\subsection{Cluster analysis of V2V interactions}
\label{ssection:clustering_quality}

\begin{figure*}[!t]
\centering
\begin{subfigure}{.22\textwidth}
\centering
\includegraphics[width = 1\textwidth]{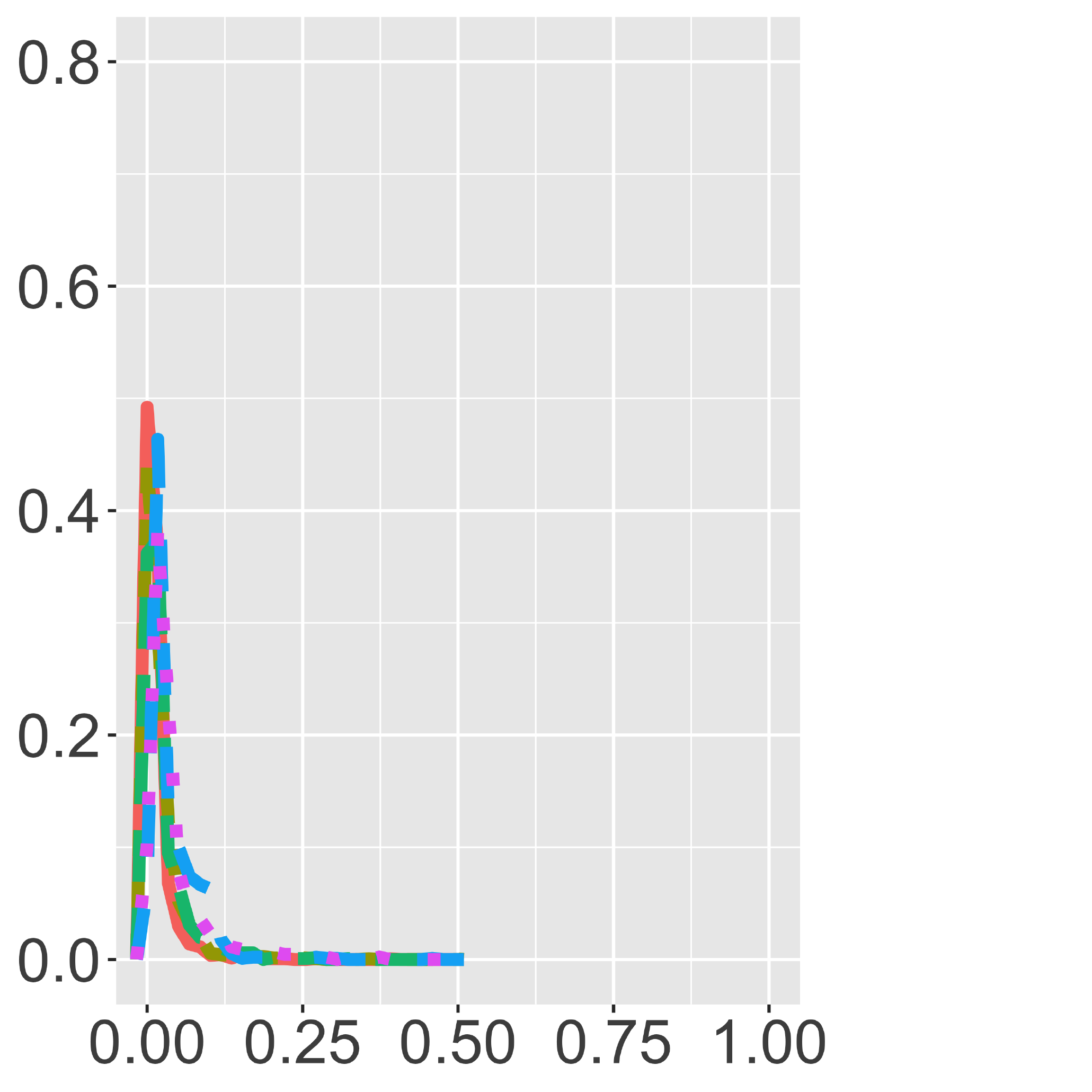}
\caption{In-cluster distances from mean interactions}
\end{subfigure}
\begin{subfigure}{.22\textwidth}
\centering
\includegraphics[width = 1\textwidth]{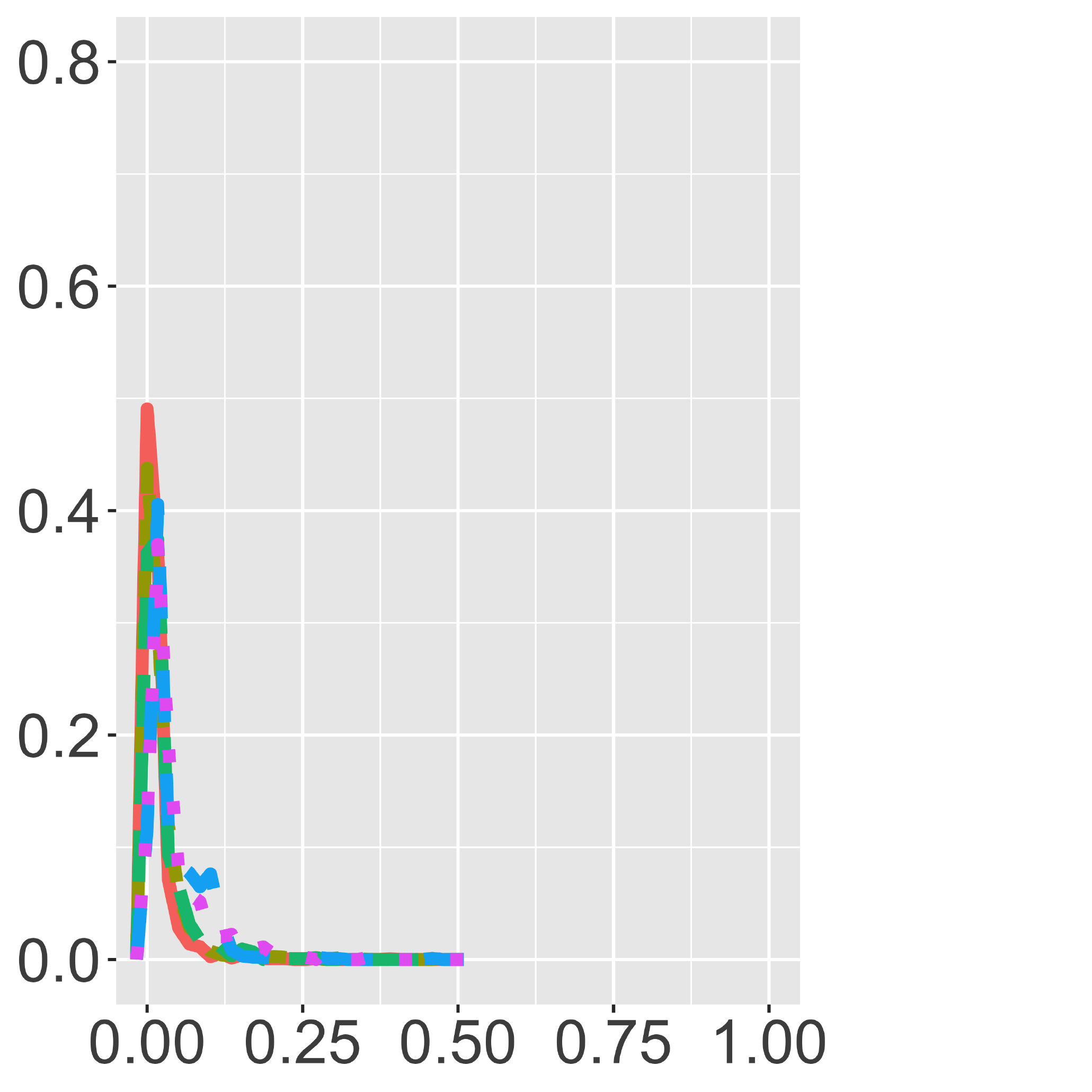}
\caption{In-cluster distances from typical interactions}
\end{subfigure}
\begin{subfigure}{.22\textwidth}
\centering
\includegraphics[width = 1\textwidth]{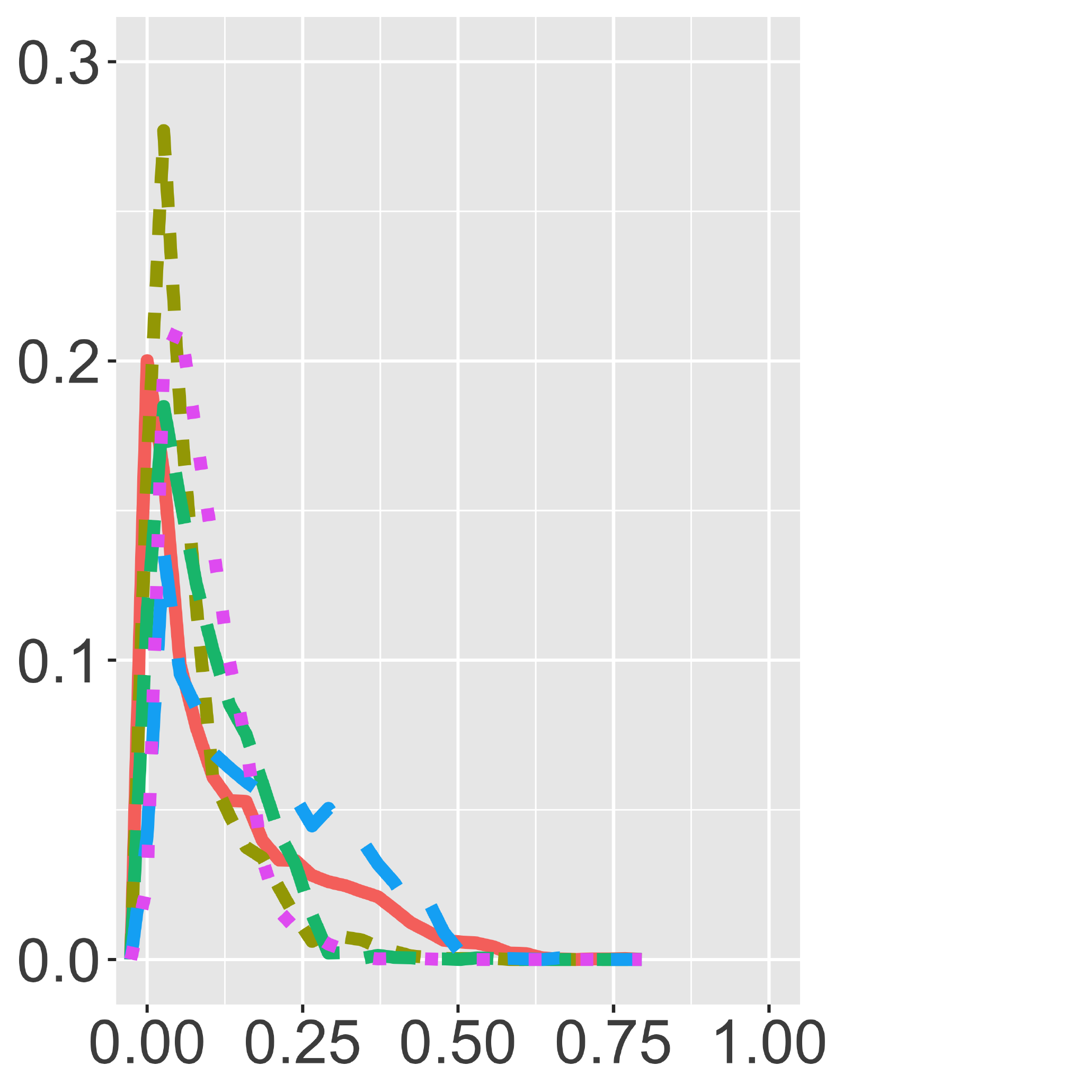}
\caption{All distances from mean interactions}
\end{subfigure}
\begin{subfigure}{.22\textwidth}
\centering
\includegraphics[width = 1\textwidth]{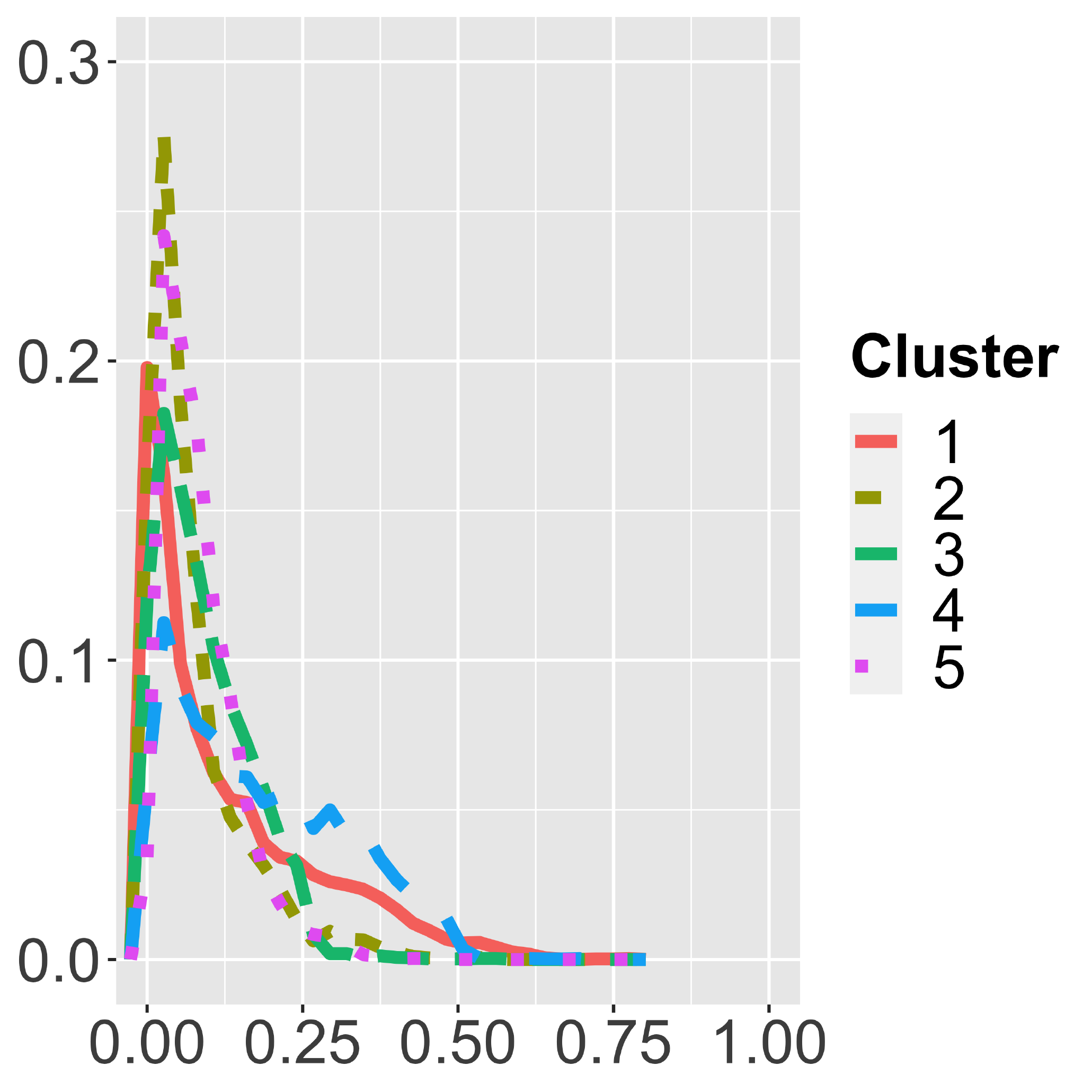}
\caption{All distances from typical interactions}
\end{subfigure}\\

\label{fig:first_geometric_approx_dist_freq_plots}

\centering
\begin{subfigure}{.22\textwidth}
\centering
\includegraphics[width = 1\textwidth]{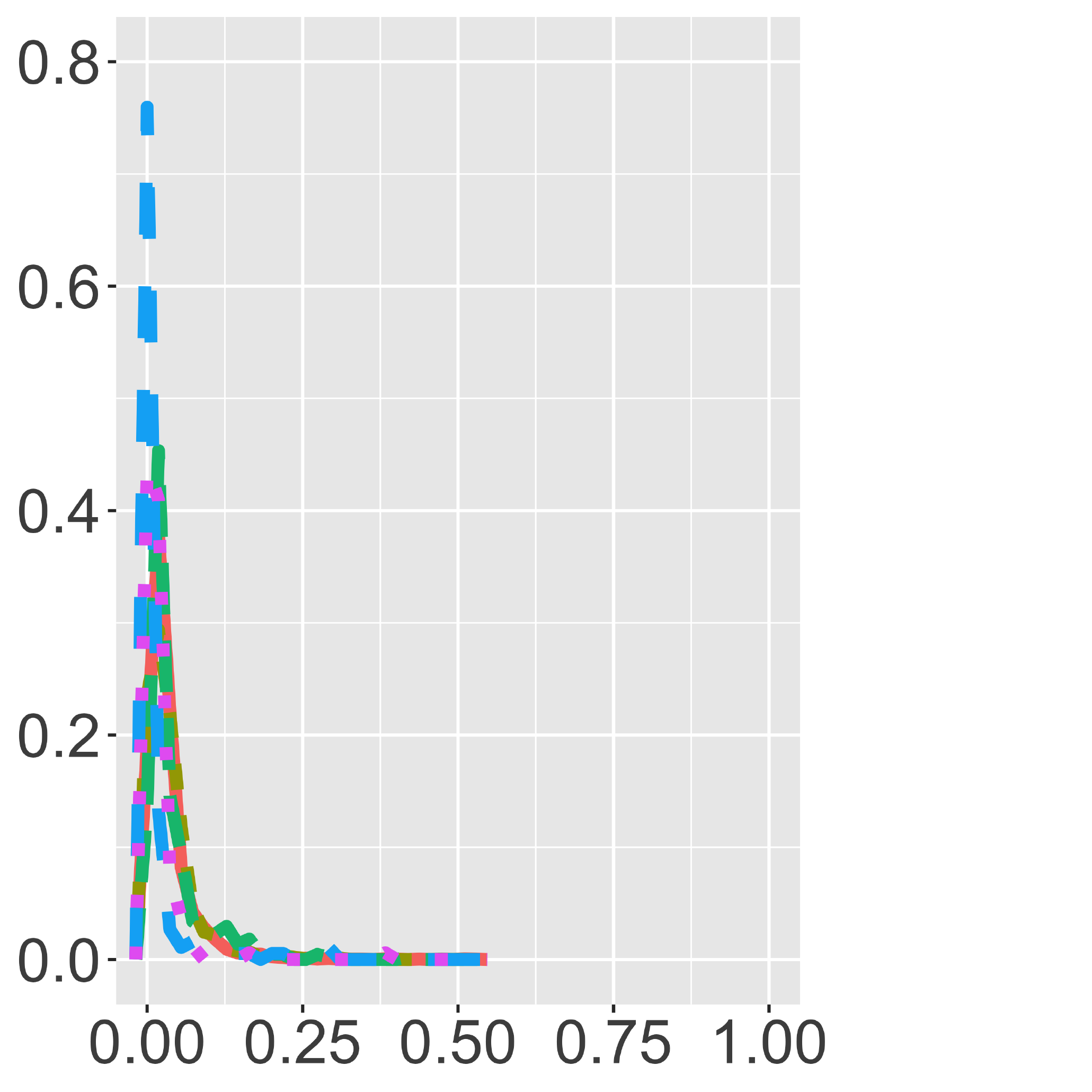}
\caption{In-cluster distances from the mean interactions}
\end{subfigure}
\begin{subfigure}{.22\textwidth}
\centering
\includegraphics[width = 1\textwidth]{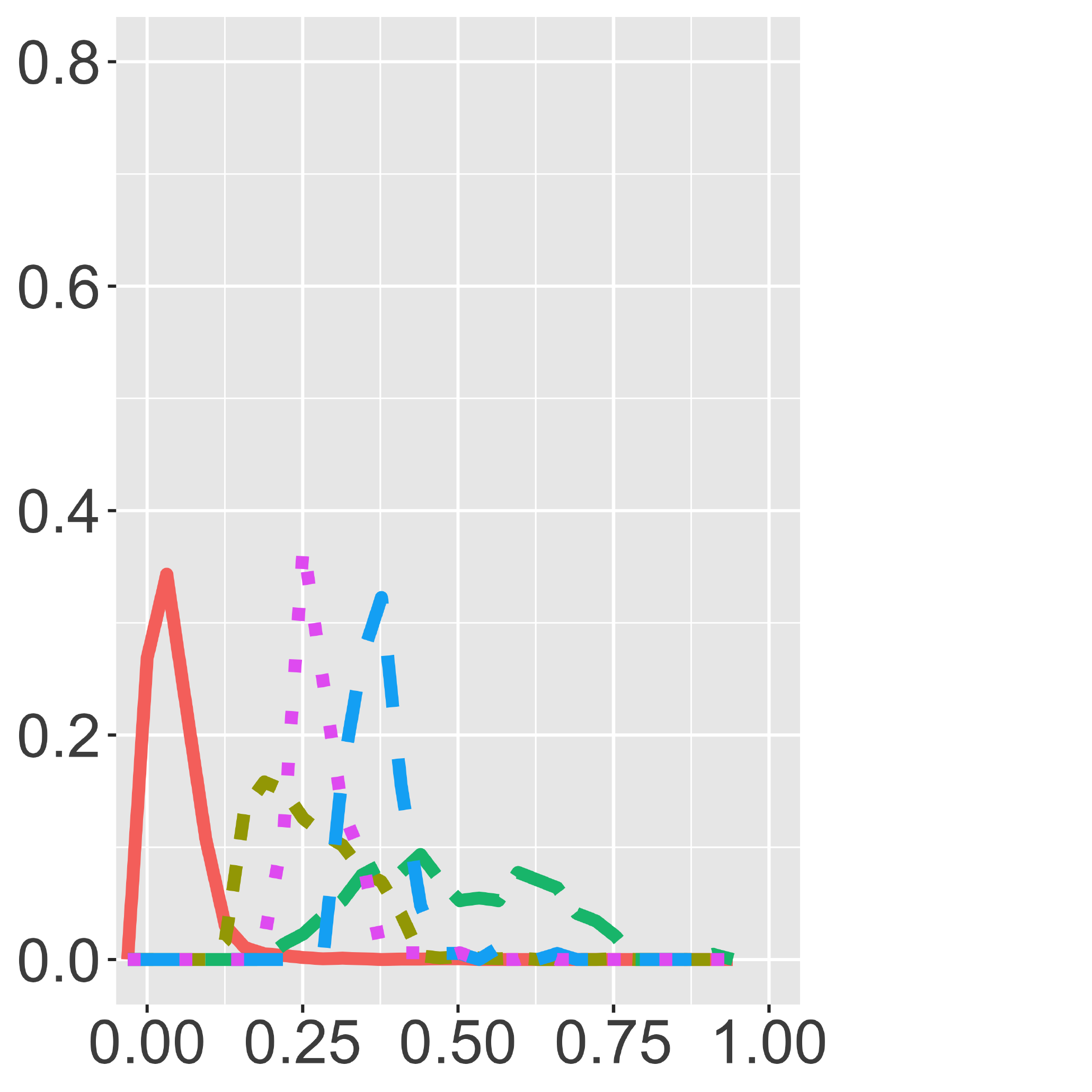}
\caption{In-cluster distances from the typical interactions}
\end{subfigure}
\begin{subfigure}{.22\textwidth}
\centering
\includegraphics[width = 1\textwidth]{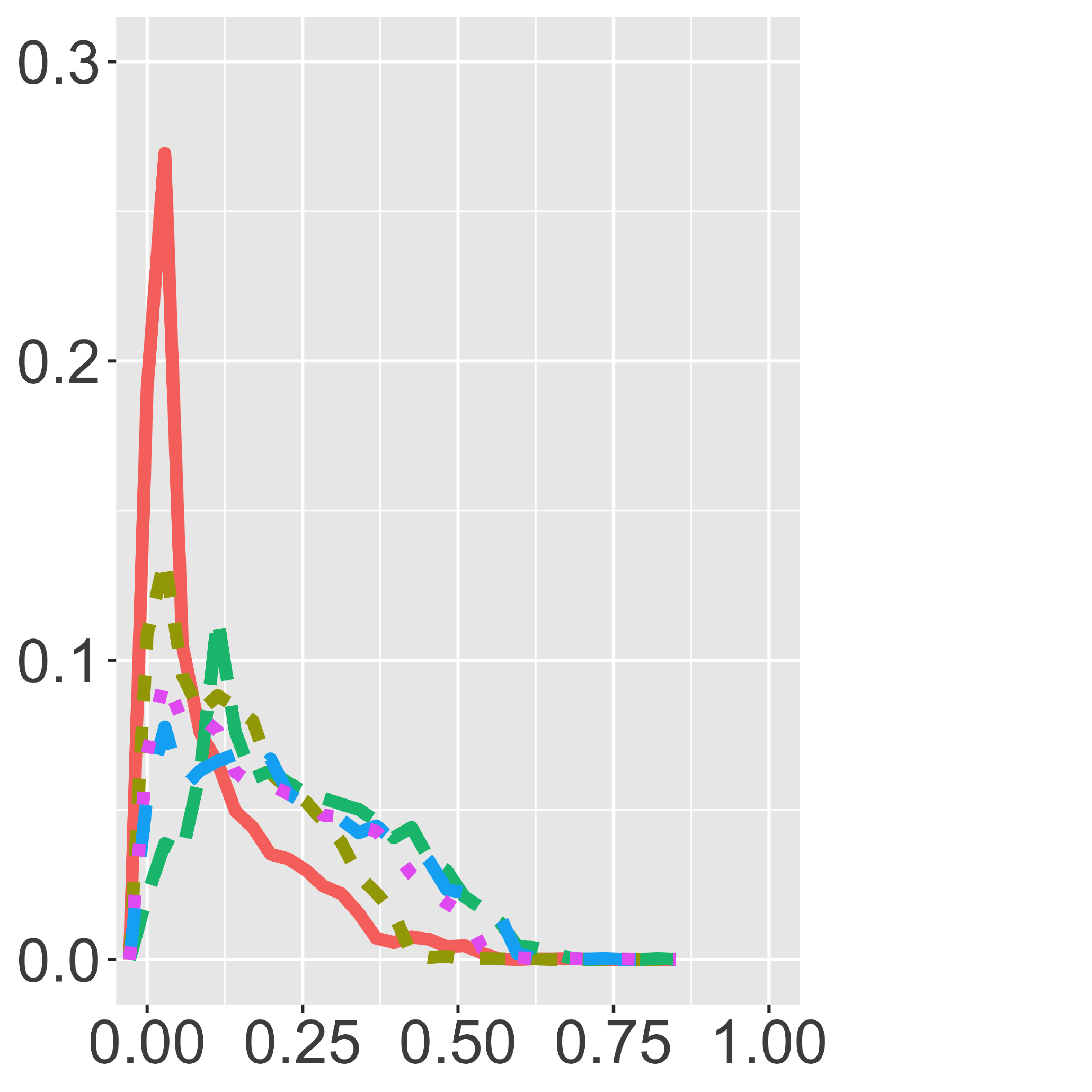}
\caption{All distances from mean interactions}
\end{subfigure}
\begin{subfigure}{.22\textwidth}
\centering
\includegraphics[width = 1\textwidth]{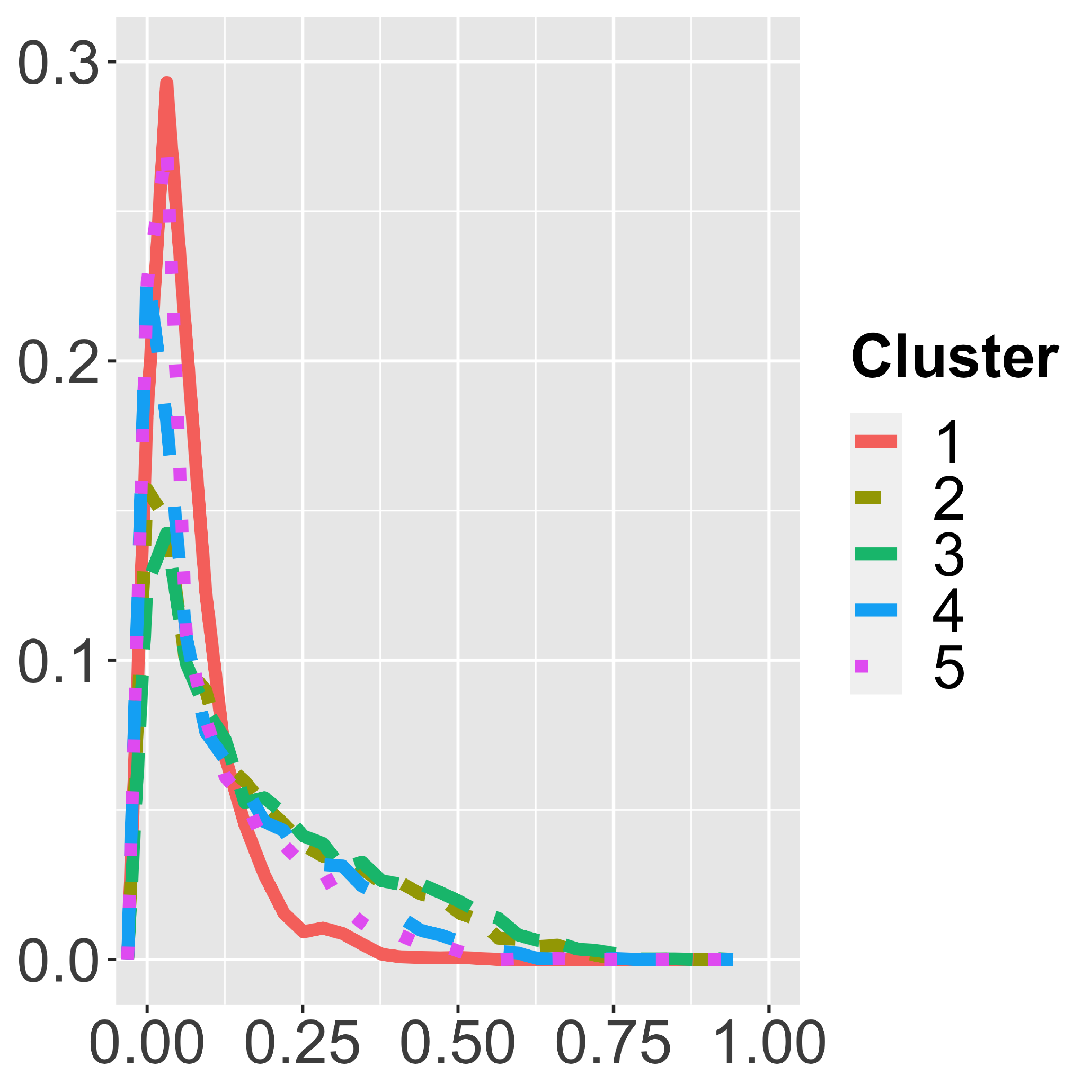}
\caption{All distances from the typical interactions}
\end{subfigure}\\

\label{fig:second_geometric_approx_dist_freq_plots}

\centering
\begin{subfigure}{.22\textwidth}
\centering
\includegraphics[width = 1\textwidth]{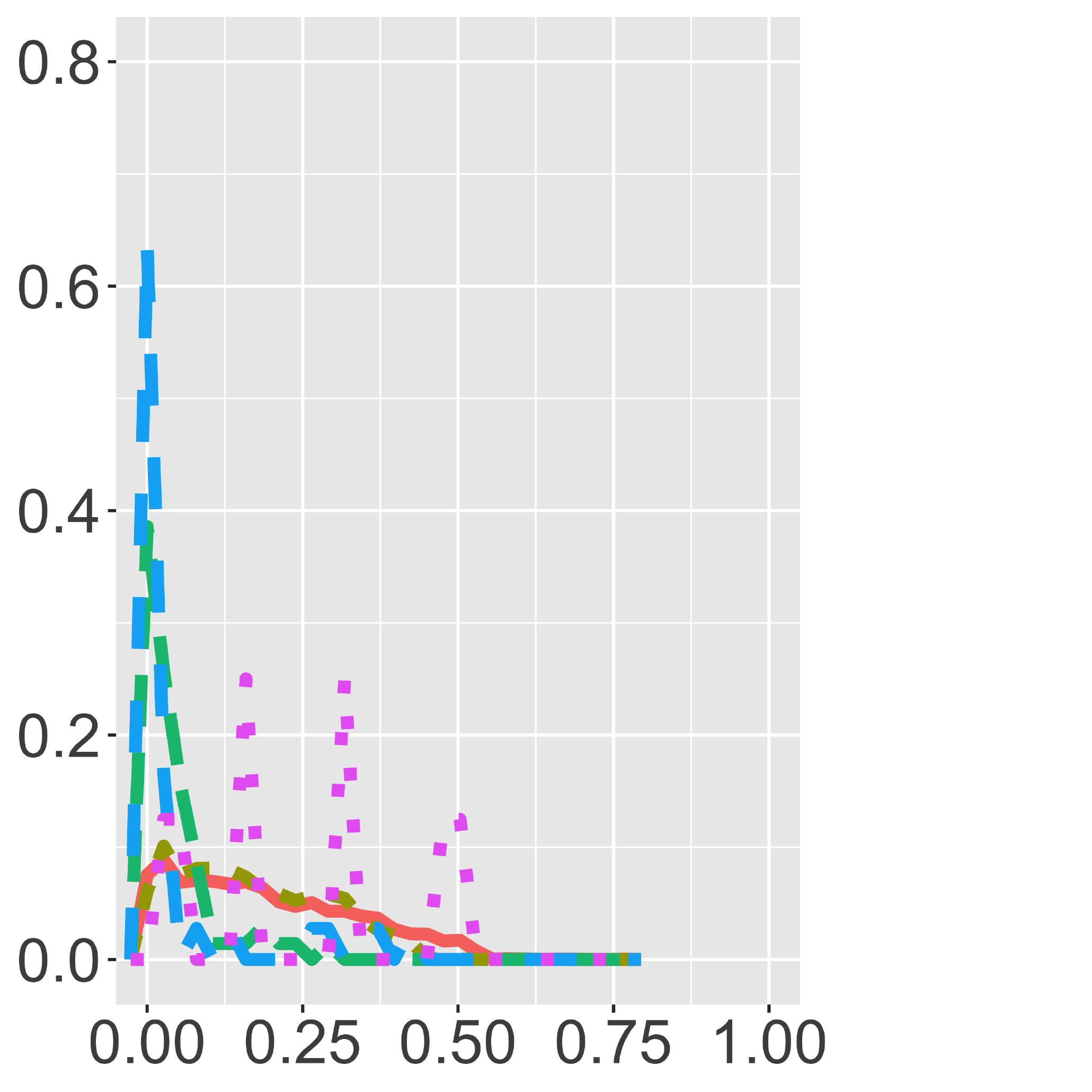}
\caption{In-cluster from mean interactions}
\end{subfigure}
\begin{subfigure}{.22\textwidth}
\centering
\includegraphics[width = 1\textwidth]{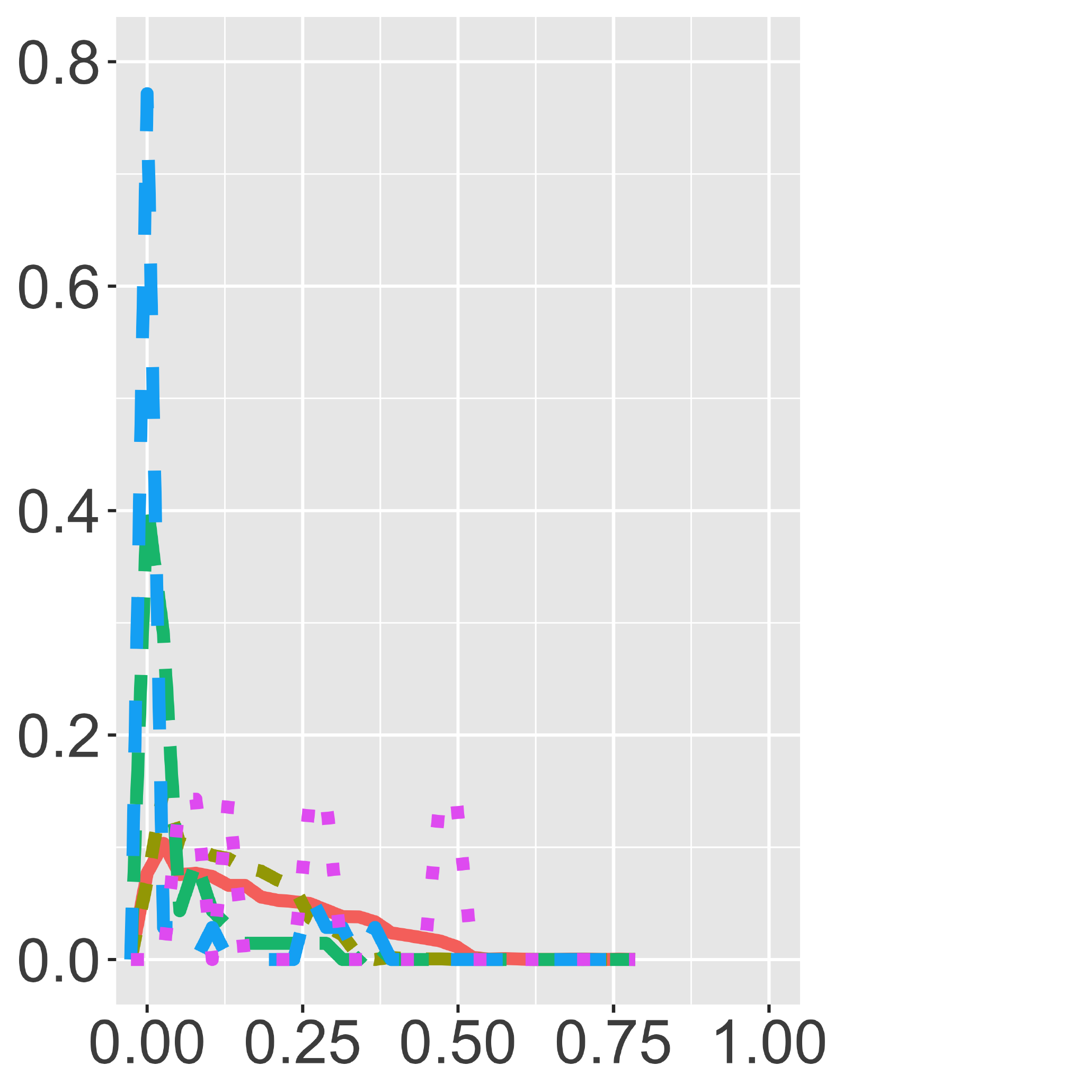}
\caption{In-cluster distances from typical interactions}
\end{subfigure}
\begin{subfigure}{.22\textwidth}
\centering
\includegraphics[width = 1\textwidth]{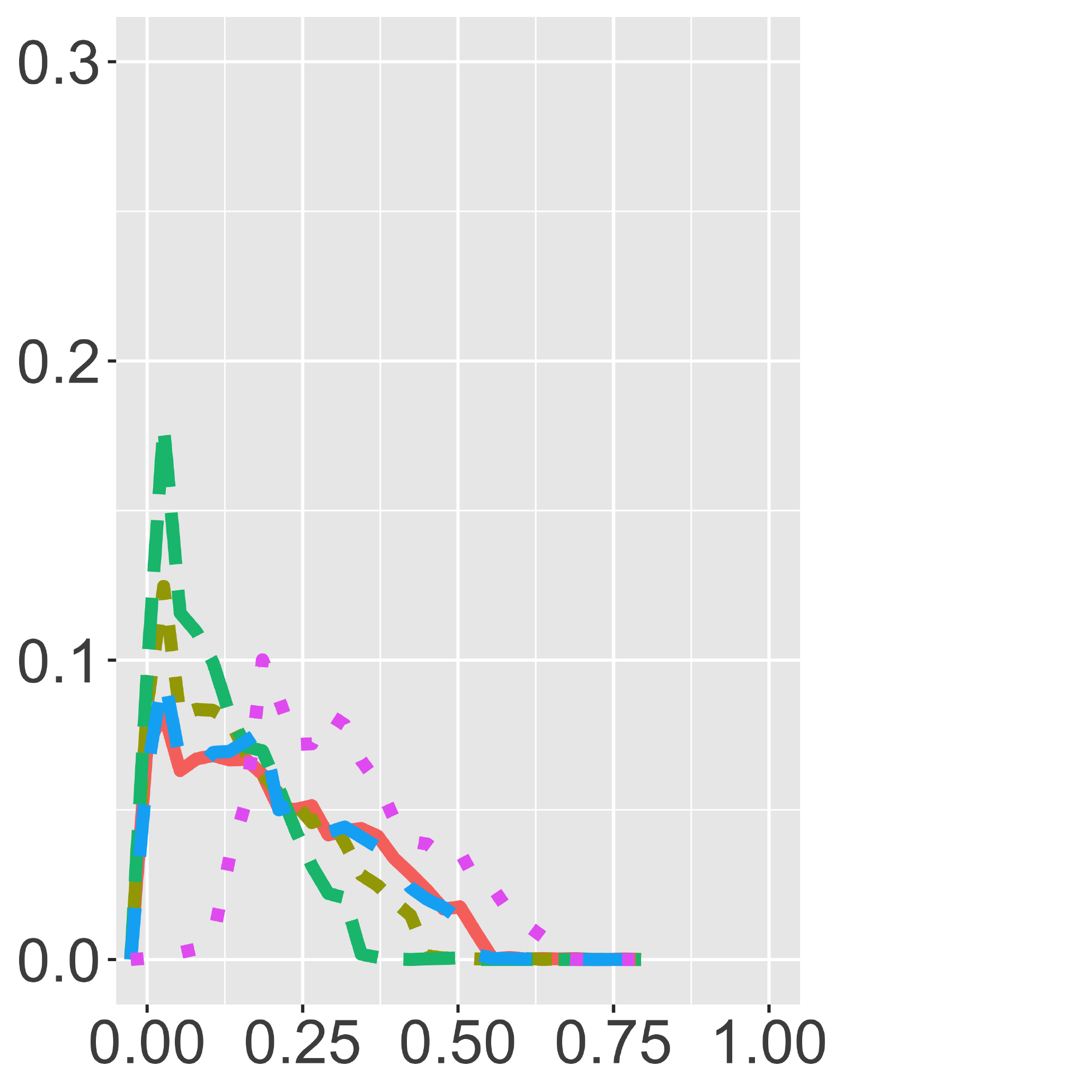}
\caption{All distances from mean interactions}
\end{subfigure}
\begin{subfigure}{.22\textwidth}
\centering
\includegraphics[width = 1\textwidth]{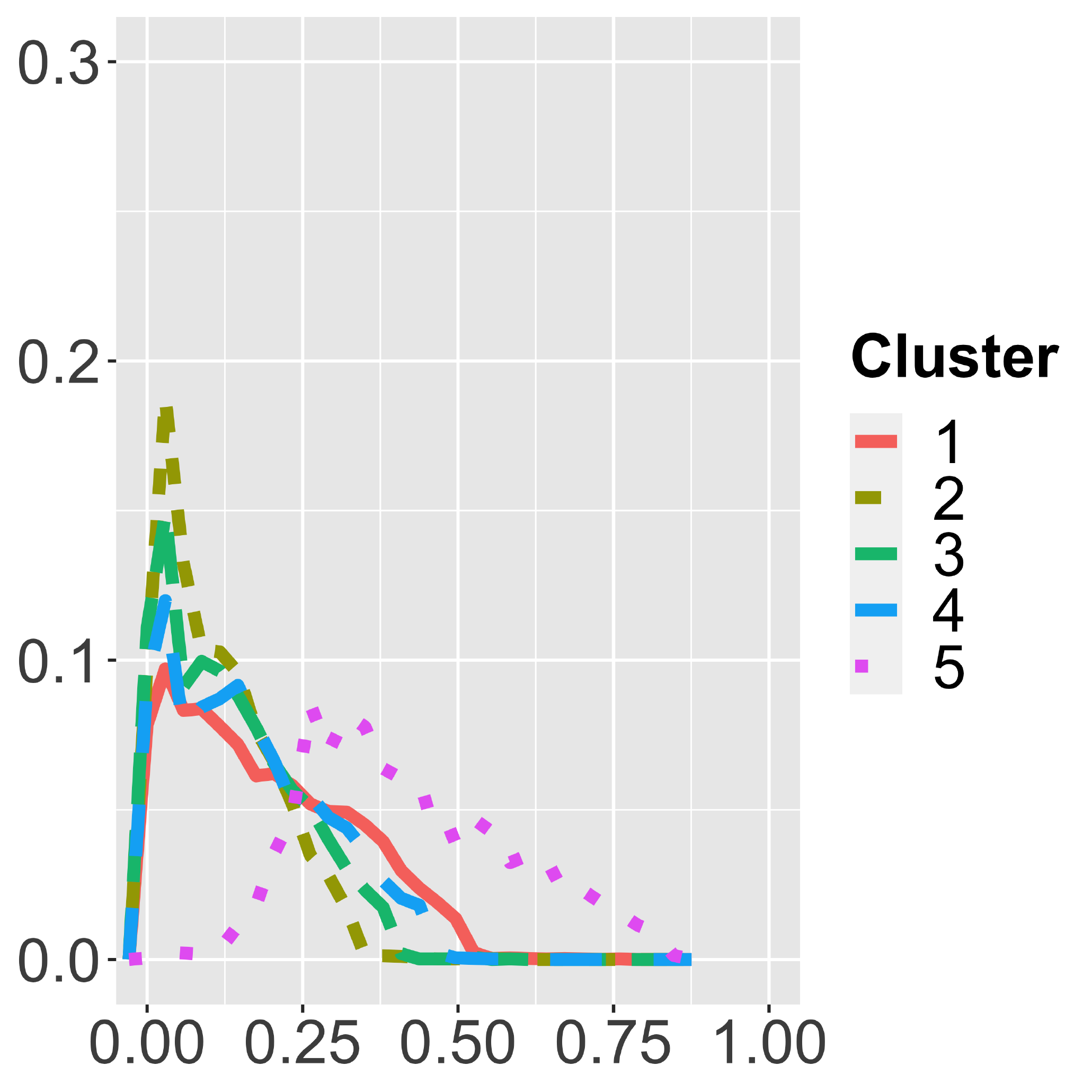}
\caption{All distances from typical interactions}
\end{subfigure}\\
\caption{Line plots showing the distribution (frequency) of interaction distance to either the cluster mean or the typical interaction. Clusters are obtained by the first geometric method in row 1, the second geometric method in row 2, and the cubic spline coefficients based method in row 3 (cf. Section~\ref{ssection:clustering_quality}). The clusters are numbered according to the number of interactions so that Cluster 1 has the most and Cluster 5 has the fewest. Note that the range for the y-axis are much larger on the left plots compared to the right plots.}
\label{fig:cluster_dist_freq_plots}
\end{figure*}

We evaluate the clustering of primitives qualitatively and quantitatively. For the former, silhouette plots are useful -- the silhouette, $s(i)$, for interaction $i$ is defined as following:
\[
s(i) = \frac{b(i) - a(i)}{\max(a(i), b(i))}.
\]
Here, $a(i)$ is the average Procrustes distance between interaction $i$ and all other interactions in the same cluster as interaction $i$, while $b(i)$ is the average Procrustes distance from interaction $i$ to those in another cluster. The cluster used for $b(i)$ is the one that minimizes this average distance. By definition, the silhouette ranges from -1 to 1. It will be close to 1 if $b(i)$ is significantly larger than $a(i)$ and -1 if $a(i)$ is significantly larger than $b(i)$. Thus, the quality of the clustering for interaction $i$ decreases as $s(i)$ decreases. Plotting the silhouettes for all interactions provides a qualitative way to determine how the clustering is performing because if most silhouettes are close to 1, the clustering is performing well. 

forming because if most silhouettes are close to 1, the clustering is performing well. 

For a more quantitative way to examine the clustering, we look at the quantity in Eq. \eqref{eq:k_means} and Eq. \eqref{eq:k_means_approx}. Naturally, the method that reduces that quantity the most should be selected. We can also break down that quantity further by the contribution of each cluster. Specifically, suppose $z_i$ indicates the cluster membership. If $(\Gamma_{j1}, \Gamma_{j2})$ is the cluster's mean, then we report the following:
\begin{align}
  \frac{1}{\countV{z_i = k}} \sum_{i:z_i = k} d^2([ (f_{i1},f_{i2})],[(\Gamma_{j1}, \Gamma_{j2})]).
  \label{eq:within_cluster_quality_stat}
\end{align} 
Alternatively, if interaction $j$ minimizes $d^2 ([ (f_{i1},f_{i2})],[ (g_{j1},g_{j2})])$ for all $z_i, z_j = k$, the approximate version is the following:
\begin{align}
  \frac{1}{\countV{z_i = k}} \sum_{i:z_i = k} d^2([ (f_{i1},f_{i2})],[(g_{j1},g_{j2})]).
  \label{eq:within_cluster_quality_stat_approx}
\end{align}
To compare clusters with different number of interactions, we choose to divide it by the size of the cluster. Finally, like the silhouette, it might also be helpful to compare this against the average square Procrustes distance of one cluster's mean V2V interaction and the interactions of all other clusters. In other words, we report the following if the cluster's mean interactions are recoverable:
\begin{align}
  \frac{1}{\countV{z_i \neq k}} \sum_{i:z_i \neq k} d^2([ (f_{i1},f_{i2})],[(\Gamma_{j1}, \Gamma_{j2})]).
  \label{eq:between_cluster_quality_stat}
\end{align}
Again, we can report the approximate version instead:
\begin{align}
  \frac{1}{\countV{z_i \neq k}} \sum_{i:z_i \neq k} d^2([ (f_{i1},f_{i2})],[(g_{j1},g_{j2})]).
  \label{eq:between_cluster_quality_stat_approx}
\end{align}
Note that the silhouette is more stringent because for the silhouette, we average only the distance from interactions of the nearest cluster for an observation.

We then made these silhouette plots and calculated the quantity for five different methods. First, we wanted to evaluate how well the three clustering approaches, namely the Multidimensional Scaling approximation and the first and second geometric approximations of the Procrustes distance, introduced in Section~\ref{Section: Distribution_primitives}, performed. Next, because we segmented encounters using splines, we wanted to examine the quality of k-means clustering based on the coefficients of the cubic splines fitted to these interactions. In other words, suppose for interaction $i$, we fit the cubic spline $c^1_{i10} + c^1_{i11}t + c^1_{i12}t^2 + c^1_{i13}t^3$ to $(g_{i1})_1$. We do the same for $(g_{i2})_2$, $(g_{i2})_1$, and $(g_{i2})_2$. Then, we perform k-means on the vectors, $\{\{c^1_{i1\ell}\}_{\ell = 0}^3, \{c^2_{i1\ell}\}_{\ell = 0}^3, \{c^1_{i2\ell}\}_{\ell = 0}^3, \{c^2_{i2\ell}\}_{\ell = 0}^3\}_{i = 1}^n$. We call this approach $\textit{spline coefficient clustering}$. Finally, dynamic time warping (DTW) is a standard approach to match curves -- in Wang et. al \cite{Wang-Zhao-2017}, k-means clustering is performed on the DTW matrices that match one trajectory to another for each V2V interaction. This is another approach we wish to evaluate. 

We focus on reporting for the case $k = 5$ for the moment, while the analysis can be replicated on other choices of $k$. The results for encounters segmented by the two step approach can be seen in Figure \ref{fig:two_step_silhouette_plots} and Table \ref{table:two_step_cluster_quality_stats_approx}.
Accordingly, the MDS approach outlined in Algorithm \ref{algo: primitive clustering} appears to perform the best whereas the spline coefficients and DTW matrices perform the worst. The total within square distance from Eq. \eqref{eq:k_means_approx} for the MDS approach in Table \ref{table:two_step_cluster_quality_stats_approx} is smallest and the silhouette plots in Figure \ref{fig:two_step_silhouette_plots} look reasonable. Indeed, even though the first geometric approximation's average within square distance for the clusters with most and fewest interactions is smaller and the average between square distance is comparable or larger, the silhouette plot shows us that the MDS approach does significantly better with the cluster with the second and third largest cluster. The silhouette values are much higher for that cluster than the first geometric approximation.

For interpretability, one may be interested in visualizing typical interactions from each clusters. Take the MDS method. There are various interesting observations in Fig. \ref{fig:two_step_typical_encounter_plot_approx_mds}. For instance, the two clusters with most interactions are interactions in which the vehicles do not move far from each other. On the other hand, the other clusters have interactions in which the opposite is true. Further, while only the cluster with the fewest interactions have vehicles going in the same direction, there are variation in how the vehicles are moving in opposite directions. Figure \ref{fig:two_step_typical_encounter_plot_approx} in the Appendix shows the three most typical encounters for all clustering methods.

One can look more deeply into the distribution of interactions in each cluster, which is revealed by Figure \ref{fig:cluster_dist_freq_plots}. The left two plots show the proportion of interactions in each cluster a certain distance away from the mean or typical interaction for each cluster. On the other hand, the right two plots show the proportion of interactions from the entire data set a certain distance away for each cluster's mean or typical interaction. The first geometric approximation plot is ideal. While the other methods have a cluster that peaks higher near zero, the left two plots show higher peaks near zero across all clusters, indicating that most interactions in the cluster are close to the typical or mean interaction. On the other hand, the right two plots show a peak near zero and then plateau for a bit before decreasing to zero. This supports what we see in the silhouette plot. The plateau demonstrates that the clusters are well separated because interactions outside the cluster are further away. Because of the left two plots, the peak near zero likely comes from the interactions assigned to that cluster. It is likely that the plots for the MDS will look similar to the first geometric approximation plots. For the second geometric approximation, the interaction plots for the mean interaction are ideal. However, outside of the largest cluster, the typical interaction to cluster interaction plots peak at values not near zero or plateau for ranges of distances. 
Meanwhile, the plots for polynomial coefficient exhibit peculiar peaks or plateaus in the left plots. These peaks are slightly dampened when using the typical interaction in place of the mean interaction. 

\subsection{Stability Evaluation}
\label{ssection:clustering_stability}

\begin{figure*}[!tbp]
\centering
\begin{subfigure}{.22\textwidth}
\centering
\includegraphics[width = 1\textwidth]{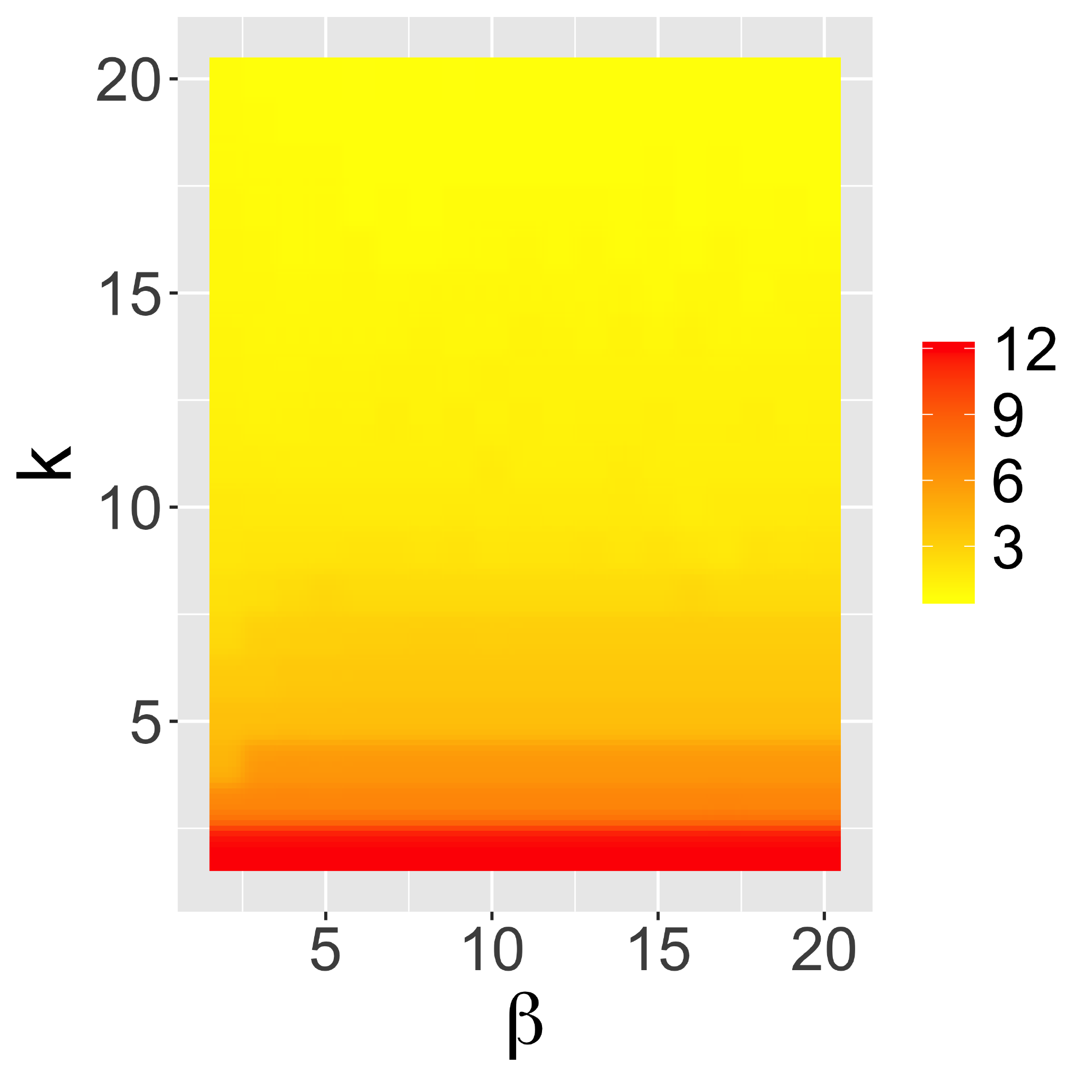}
\caption{All $k$ and $\beta$ between 2 and 20}
\end{subfigure}
\begin{subfigure}{.22\textwidth}
\centering
\includegraphics[width = 1\textwidth]{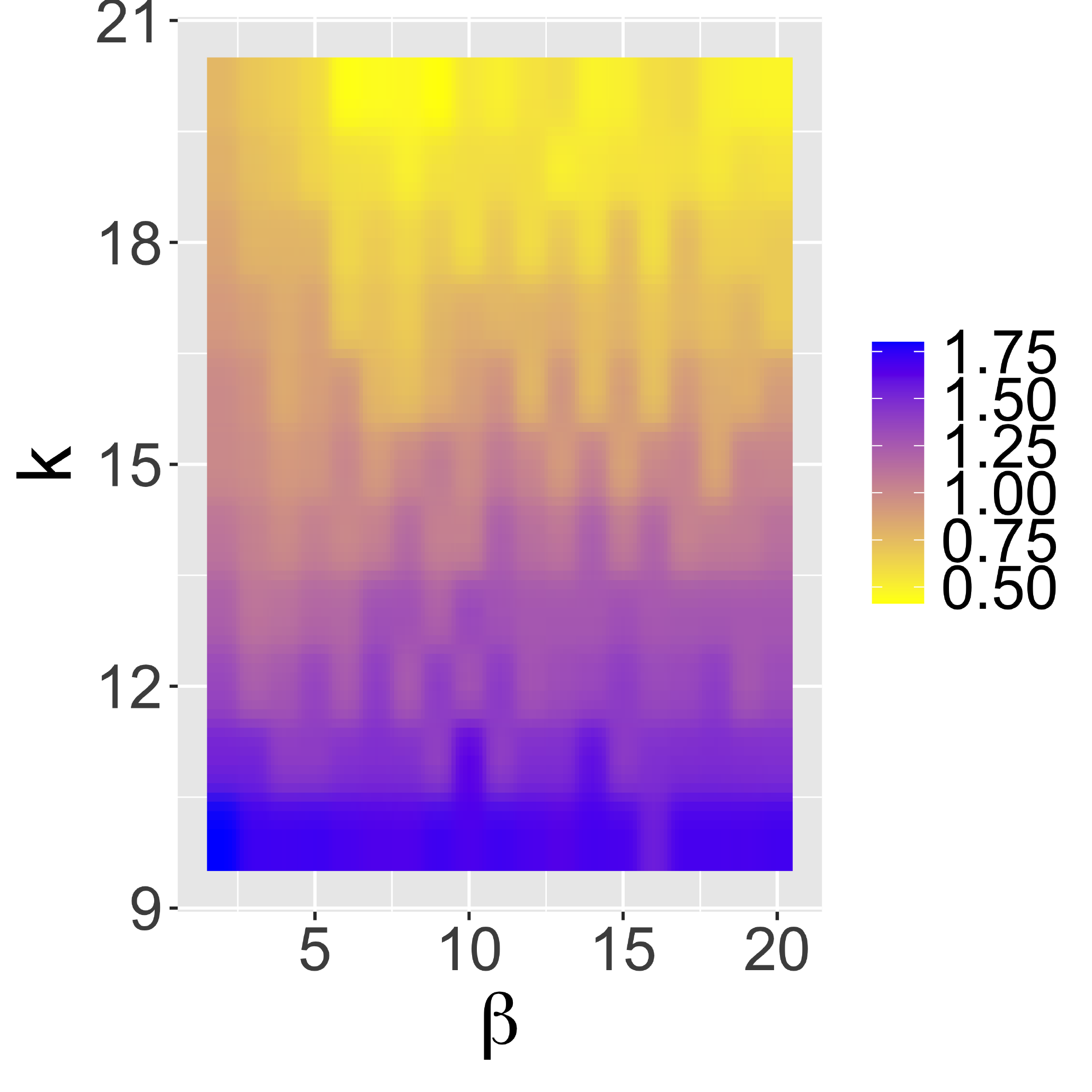}
\caption{All $\beta$ between $10$ and $20$ and $k$ between $10$ and $20$}
\end{subfigure}
\centering
\begin{subfigure}{.22\textwidth}
\centering
\includegraphics[width = 1\textwidth]{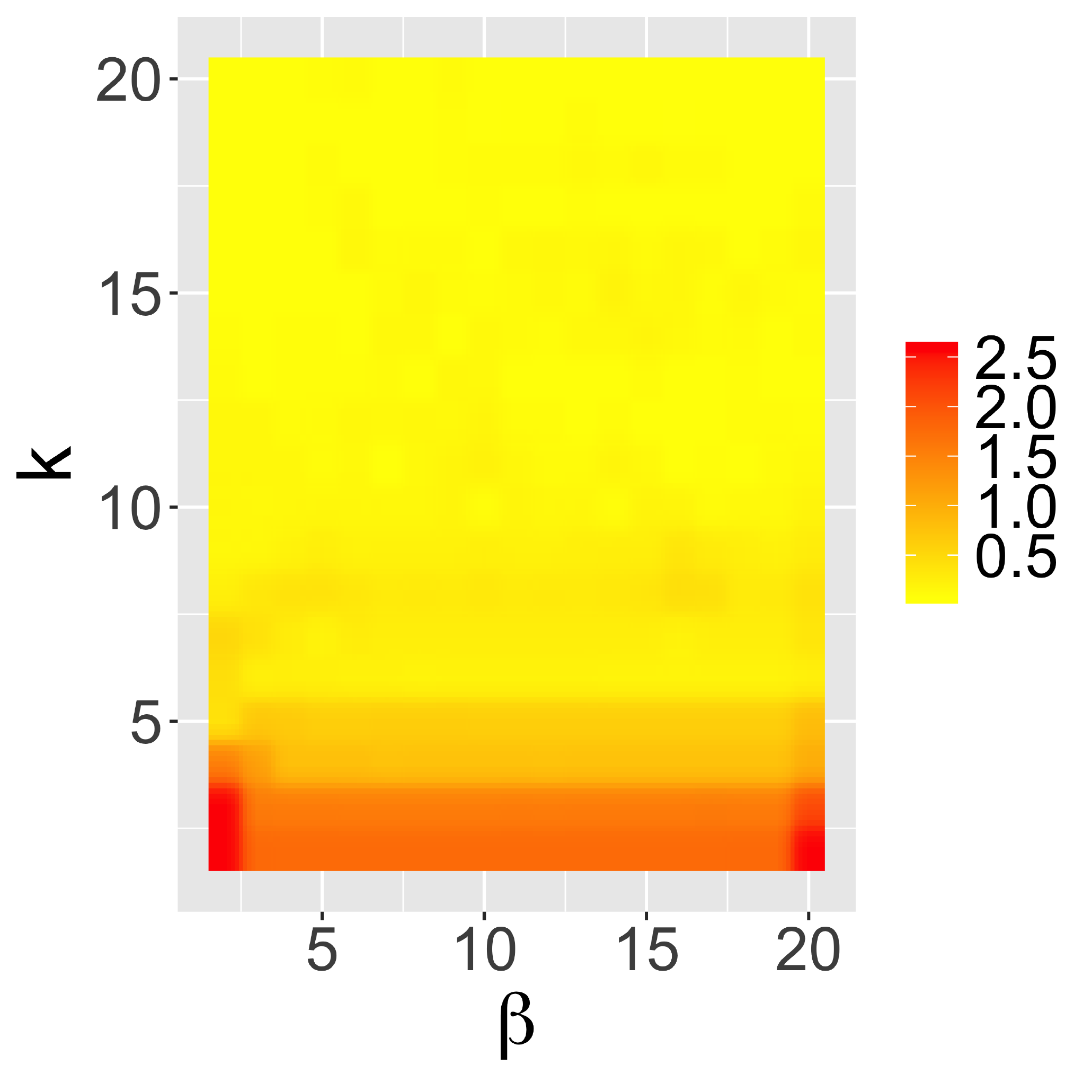}
\caption{All $\beta$ between 2 and 20 and $k$ between 10 and 20}
\end{subfigure}
\begin{subfigure}{.22\textwidth}
\centering
\includegraphics[width = 1\textwidth]{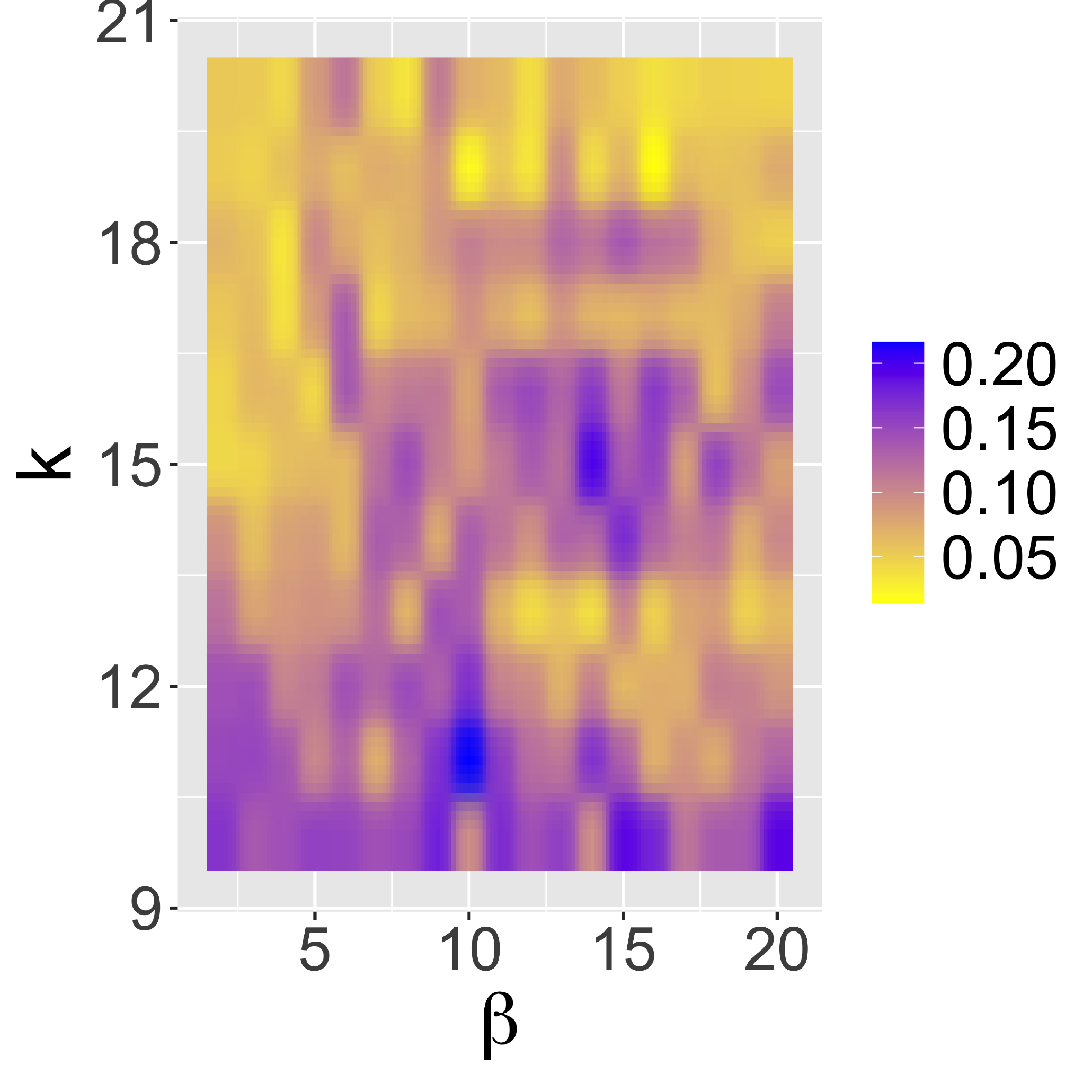}
\caption{All $\beta$ between $2$ and $20$ and $k$ between $10$ and $20$}
\end{subfigure}
\caption{The left two plots show heatmaps of the statistic introduced in Eq. \eqref{eq:equivalent_kmeans problem} for encounters segmented by the two-step spline approach (cf. Appendix \ref{Section: Spline}) and then clustered via MDS (cf. Section \ref{ssection:approximation}). The right two show changes in this statistic.}
\label{fig:eq_kmeans_heat_map_no_reflection}

\centering
\begin{subfigure}{.3\textwidth}
\centering
\includegraphics[width = .75\textwidth]{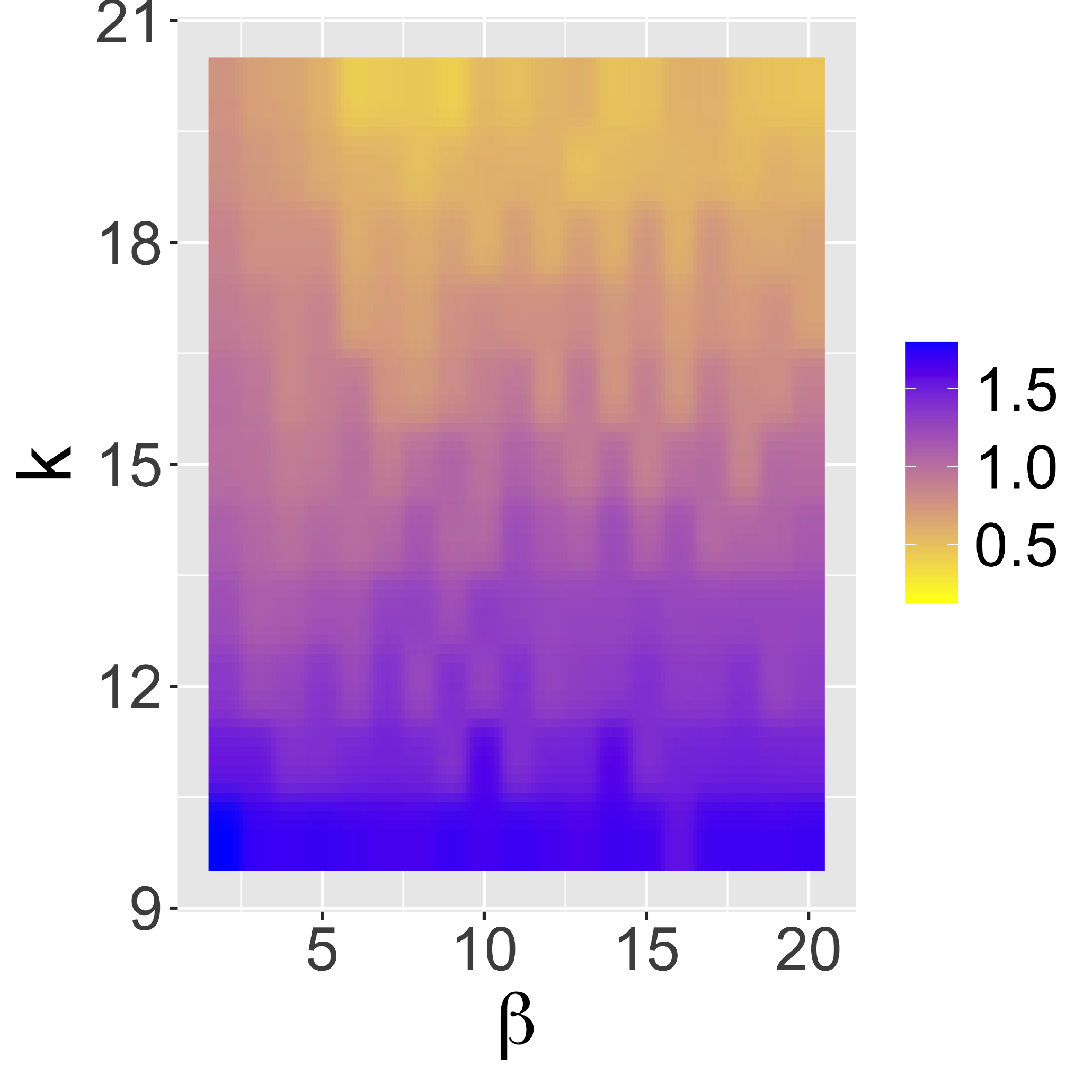}
\caption{MDS clustering (cf. Section \ref{ssection:approximation}) applied to two-step spline segmented encounters (cf. Appendix \ref{Section: Spline}).}
\end{subfigure}
\begin{subfigure}{.3\textwidth}
\centering
\includegraphics[width = .75\textwidth]{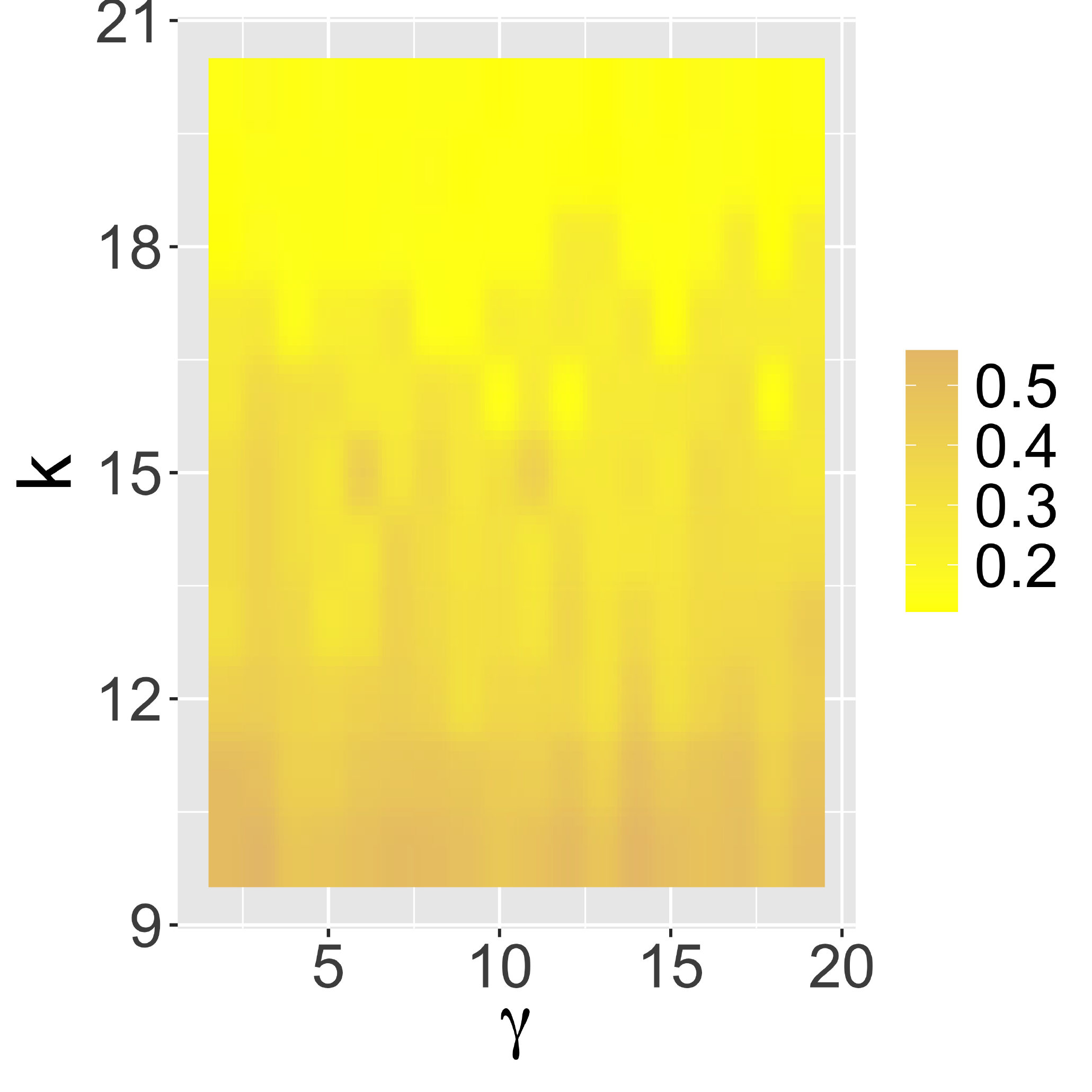}
\caption{DTW matrix clustering applied to encounters segmented by BNP (cf. \cite{Wang-Zhao-2017}).\\}
\end{subfigure}
\begin{subfigure}{.3\textwidth}
\centering
\includegraphics[width = .75\textwidth]{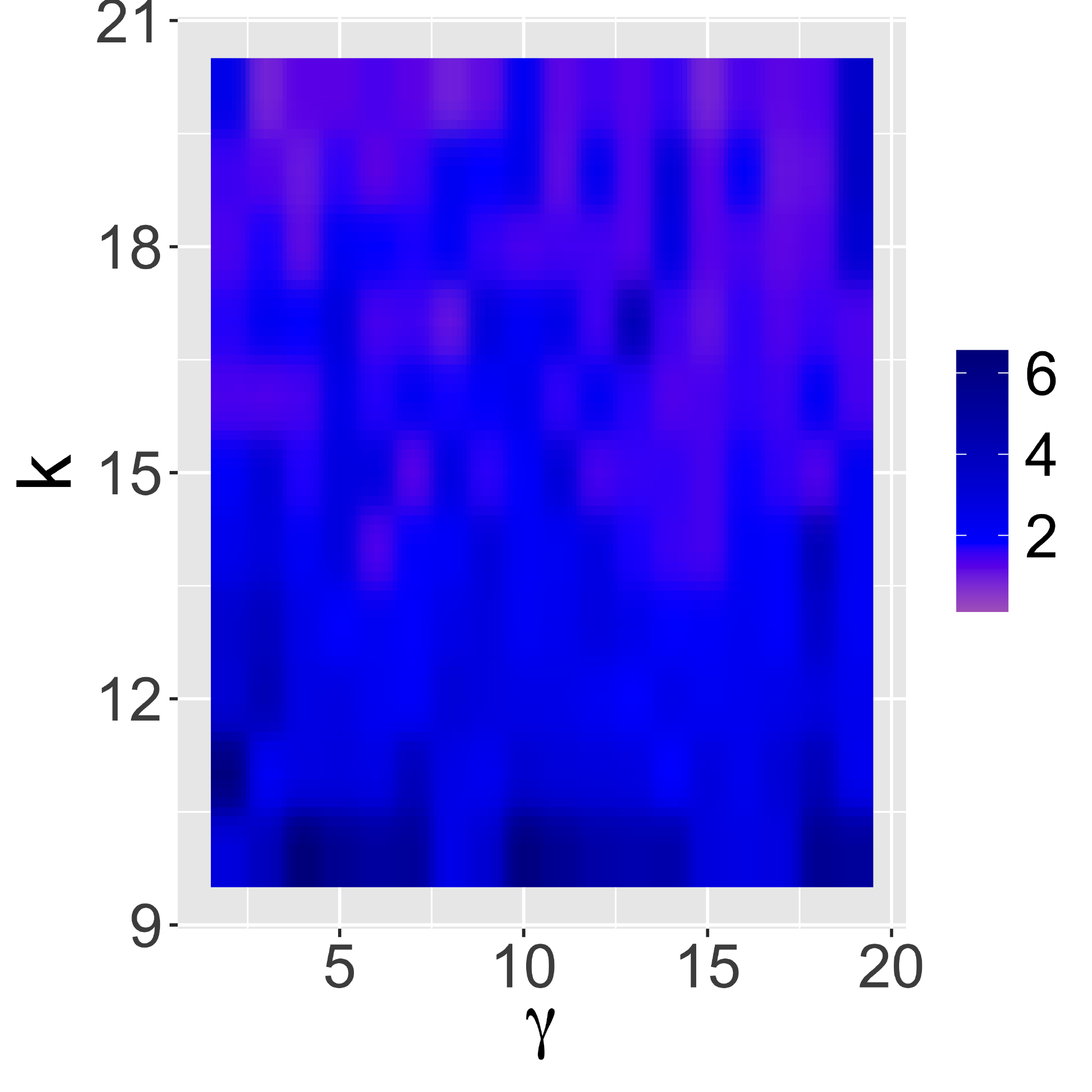}
\caption{Two step spline segmented encounters clustered using primitives extracted from BNP segmented encounters clustered using DTW matrices (cf. \ref{ssection:clustering_stability}).}
\end{subfigure}
\caption{Heatmaps of the statistic given by Eq. \eqref{eq:equivalent_kmeans problem} for non-reflective Procrustes distance for $k \geq 10$ across different methods.}
\label{fig:k_gt_10_stability_comparison}
\end{figure*}

We had to develop a statistic for stability based on our distance metric. Consider the k-means problem in a Euclidean space. Let $x_1,\dots,x_n \in \mathbb{R}^d$ be points in an Euclidean space belonging to clusters $\{1,\dots,K\}$. The cost function relative to the k-means problem is given by 
\begin{eqnarray}
\min_{\{\Gamma_1,\dots,\Gamma_k, z_1,\dots,z_n\}}\frac{1}{n}\sum_{j=1}^k \sum_{i:z_i=j} \| x_i - \Gamma_j\|^2.\nonumber
\end{eqnarray}
The above cost function can also be written as:
\begin{eqnarray}
\label{eq:kmeans problem}
\min_{\{z_1,\dots,z_n\}}\frac{1}{2n}\sum_{k=1}^K \sum_{i,j:z_i,z_j=k} \| x_i-x_j\|^2.
\end{eqnarray}
Eq.~\eqref{eq:kmeans problem} provides a way to partition the dataset $\{x_1,\dots,x_n\}$ so as to optimize the within cluster distance. We then use a measure equivalent to~\eqref{eq:kmeans problem} to evaluate the stability of algorithms. Namely, if we use the same notation as before, then the stability of the algorithm is measured by computing

\begin{eqnarray}
\label{eq:equivalent_kmeans problem}
\frac{1}{2n}\sum_{k=1}^K \sum_{i,j:z_i,z_j=k} d^2((f_{i1}, f_{i2}),(f_{j1}, f_{j2}))),
\end{eqnarray}
for varying values of tuning parameters, where $d((f_{i1}, f_{i2}),(f_{j1}, f_{j2})))$ is the metric introduced in Eq.~\eqref{eq:metric_quotient_space}.

We then calculated this statistic for the MDS approach outlined in Section \ref{ssection:approximation}. Applying k-means to the MDS projection of the (non-reflective) Procrustes distance requires the specification of the following parameters: dimension of the projection, $\beta$, and the number of clusters, $k$. The results for $\beta \in [2, 20]$ and $k \in [2, 20]$ can be seen in Figure \ref{fig:eq_kmeans_heat_map_no_reflection}. From the heatmap, the MDS approach is particularly stable for $k \geq 15$ and $\beta > 5$. This is further supported by examining the change in the statistic introduced in Eq. \eqref{eq:equivalent_kmeans problem}. This makes sense because increasing the dimension for the MDS representation provides a better representation of the pairwise distance. On the other hand, increasing $k$ also leads to greater stability. 
However, the scales in the figure suggests that most of the instability occurs when $k \leq 10$.

As before, we proceed to compare among methods and data sets. First,
following Wang and Zhou \cite{Wang-Zhao-2017}, we examined the stability of the DTW approach for the encounters segmented by sticky HMM-HDP. While there are more parameters to consider, we empirically investigated the results with $\alpha$ and $c$ fixed to 2 and 100 respectively and allowed $\gamma$ and $k$ to vary between $[2, 19]$ and $[2, 20]$. Because changing $\gamma$ gives us new primitives, we had to interpolate and recalculate the Procrustes distance for each set of primitives. Second, we wanted to inspect the stability of "transferring" primitives. In other words, let $\{(g'_{j1}, g'_{j2})\}_{j = 1}^k$ be the primitives derived from applying BNP to segment encounters and using the DTW matrices to cluster them and $\{(f_{i1}, f_{i2})\}_{i = 1}^{n}$ be the interactions extracted from the encounters using our two-step approach outlined in Appendix \ref{Section: Spline}. We assign interaction $i$ to cluster $j$ if 
\[
j = \textrm{argmin}_{j' = 1:k} d((f_{i1}, f_{i2}), (g'_{j1}, g'_{j2})).
\]
Here, $d((f_{i1}, f_{i2}), (g'_{j1}, g'_{j2}))$ is the distance introduced in Eq. \eqref{eq:rotation_translation_invariance}. 
The results can be seen in Figure \ref{fig:k_gt_10_stability_comparison}. For DTW, the results are similar to the results before with respect to $k$ and may even be better. On the other hand, as seen in the scales in Figure \ref{fig:k_gt_10_stability_comparison}, we see that there is greater instability in both the range and the pattern when we "transfer" primitives. Further, unlike before, this instability persists even as $k$ increases. This could be due to the more extreme values in the BNP primitive data set. As a result, there might be primitives that do not exist in the data set segmented by the two step spline approach. This could mean that as we increase $k$, we might not be adding centroids used to cluster the data. In addition, the ones that do exist might be influenced by these more extreme values. This might be why the values are unstable for lower values of $k$.

\section{Conclusion}
We developed a distance metric for the space of trajectory pairs that is invariant under translation and rotation. By using it to measure the distance between distributions, we could also use this metric for clustering and for evaluating a variety of unsupervised techniques for interaction learning. The distance metric and geometric approximation methods that we introduced help to address the challenges for robust learning of non-Euclidean quantities that represent temporally dynamic interactions. These techniques were demonstrated by the unsupervised learning of vehicle-to-vehicle interactions. An interesting direction for our work is to extend the metric based representation and geometric algorithms to the multiple-vehicle interaction setting, and general multi-agent settings. The challenge is the find a right metric or a family of metrics which are both meaningful and computationally tractable for a number of learning tasks of interests.

\section*{Acknowledgements}
Toyota Research Institute (TRI) provided funds to assist the authors with their research but this article solely reflects the opinions and conclusions of its authors and not TRI or any other Toyota entity. Dr. Nguyen is also partially supported by grants NSF CAREER DMS-1351362 and NSF CNS-1409303.

\bibliography{Aritra_driving,Aritra_BNP}

\newpage
\section*{Appendix}
\label{Section: Appendix}
\addcontentsline{toc}{section}{Appendices}
\renewcommand{\thesubsection}{\Alph{subsection}} 

\subsection{Proofs}
\subsubsection{Proof of Proposition~\ref{proposition:proper_metric}}
\label{proof:proposition::proper_metric}
We need to establish
\begin{enumerate}
    \item[ (a)] For any $f_{11},f_{12},f_{21},f_{22} \in \Fbb$,$d ( (f_{11},f_{12}), (f_{21},f_{22}))=0 $ if and only if $ (f_{11},f_{12}) \sim  (f_{21},f_{22})$.
    \item[ (b)] For any $f_{11},f_{12},f_{21},f_{22} \in \Fbb$, $d ( (f_{11},f_{12}), (f_{21},f_{22}))=d ( (f_{21},f_{22}), (f_{11},f_{12})) $.
    \item[ (c)] For any $f_{11},f_{12},f_{21},f_{22},f_{31},f_{32} \in \Fbb$, $d ( (f_{11},f_{12}), (f_{21},f_{22}))  \\ \leq d ( (f_{31},f_{32}), (f_{21},f_{22})) +d ( (f_{11},f_{12}), (f_{31},f_{32})) $.
\end{enumerate}

Condition  (a) follows by definition. To establish (b), note $\rho ( (f_{11},f_{12}), (f'_1,f'_2))=\rho ( (f'_1,f'_2), (f_{11},f_{12}))$, so
\begin{eqnarray}
\label{eq:transpose}
 & & \rho ( (f_{11},f_{12}), O_1 \odot(f_{21},f_{22}) + c_1) \nonumber \\ &=&  \rho (O_1 \odot(f_{21},f_{22}) + c_1,  (f_{11},f_{12}))  \\
  &=& \rho ((f_{21},f_{22}), O_1^* \odot(f_{11},f_{12}) - O_1^* \odot c_1) \nonumber,
\end{eqnarray}
where the second equality is due to property (i) and (ii) in the proposition, with $O_1^*$ being the conjugate transpose of $O_1$, which is also orthogonal when $O_1$ is. Now taking infimum over $C_1$ and $O_1$ the conclusion of part (b) is achieved. 

For condition (c), notice that it is easy to see, following the argument similar to Eq.~\eqref{eq:transpose}, that
\begin{eqnarray}
 & &\inf_{ O_1,O_2 \in SO (2); C_1,C_2\in \mathbb{R}^2} \rho (O_2 \odot (f_{11},f_{12})   +C_2, O_1 \odot(f_{21},f_{22})+C_1) 
 \\ &=& \inf_{ O_1 \in SO (2), C_1\in \mathbb{R}^2 } \rho ((f_{11},f_{12}), O_1 \odot(f_{21},f_{22})+C_1). \nonumber
\end{eqnarray}

Now for any $f_{31},f_{32} \in \Fbb$,
\begin{eqnarray}
 & & \rho (O_2 \odot (f_{11},f_{12}) +C_2,  O_1 \odot(f_{21},f_{22})+C_1)  \\ &\leq& \rho (O_2 \odot (f_{11},f_{12}) +C_2,(f_{31},f_{32})) + \rho ((f_{31},f_{32}), O_1 \odot(f_{21},f_{22})+C_1), \nonumber 
\end{eqnarray}
by triangle inequality applied to $\rho$. Taking infimum wrt $O_1,O_2 \in SO (2); C_1,C_2\in \mathbb{R}^2$, the rest follows immediately.
\subsubsection{Proof of Proposition~\ref{proposition:computation}}
\label{proof:proposition::computation}
Note that 
\begin{eqnarray}
\label{eq:distance_Calculation}
 & &d( (f_{11},f_{12}), O \odot(f_{21},f_{22}) +c))^2 :=  \\ & & \int_0^{\infty} \biggr(\|f_{11} (x) - O \cdot f_{21} (x)-c\|_2^2 + \|f_{12} (x) - O \cdot f_{22} (x) -c \|_2^2\biggr) \mu (\mathrm{d}x). \nonumber
\end{eqnarray}
Minimizing Eq.~\eqref{eq:distance_Calculation} with respect to $c$, for fixed $O$, we get

\begin{eqnarray}
   c &=& \int_0^\infty \frac{f_{11} (x) + f_{12} (x)}{2}\mu (\mathrm{d}x) - O \cdot\left (\int_0^\infty \frac{f_{21} (x) + f_{22} (x)}{2}\mu (\mathrm{d}x)\right)\nonumber. 
\end{eqnarray}
Substituting this value of $c$, we obtain Eq.~\eqref{eq:minimize}.

\begingroup
\small
\begin{equation}
\begin{split}
\label{eq:minimize}
&\inf_{ c\in \mathbb{R}^2}(\rho ( (f_{11},f_{12}), O \odot(f_{21},f_{22}) +c))^2   \\ &=-2\int_0^\infty \left (f_{11} (x)-\int_0^\infty \dfrac{f_{11} (x) + f_{12} (x)}{2}\mu (\mathrm{d}x)\right)^T \cdot O \cdot \left (f_{21} (x)-\int_0^\infty \frac{f_{21} (x) + f_{22} (x)}{2}\mu (\mathrm{d}x)\right)\mu (\mathrm{d}x) \\ &- 2 \int_0^\infty \left (f_{12} (x)-\int_0^\infty \dfrac{f_{11} (x) + f_{12} (x)}{2}\mu (\mathrm{d}x)\right)^T  O \cdot \left (f_{22} (x)-\int_0^\infty \frac{f_{21} (x) + f_{22} (x)}{2}\mu (\mathrm{d}x)\right)\mu (\mathrm{d}x)\\ &+ \sum_{i=1}^2\int_0^\infty \biggr\|f_{2i} (x)-\int_0^\infty \frac{f_{21} (x) + f_{22} (x)}{2}\mu (\mathrm{d}x)\biggr\|_2^2 \mu (\mathrm{d}x) + \sum_{i=1}^2\int_0^\infty \biggr\|f_{1i} (x)-\int_0^\infty \frac{f_{11} (x) + f_{12} (x)}{2}\mu (\mathrm{d}x)\biggr\|_2^2\mu (\mathrm{d}x) . 
\end{split}
\end{equation}
\endgroup

Minimizing Eq.~\eqref{eq:minimize} with respect to $O$ is same as maximizing Eq.~\eqref{eq:maximize} with respect to $O \in SO (2)$. 
\begingroup
\small
\begin{equation}
\begin{split}
\label{eq:maximize}
& 2 \textrm{trace} \biggr( \int_0^\infty \biggr(f_{11} (x)-\int_0^\infty \dfrac{f_{11} (x) + f_{12} (x)}{2}\mu (\mathrm{d}x)\biggr)^T \cdot O \cdot \left (f_{21} (x)-\int_0^\infty \frac{f_{21} (x) + f_{22} (x)}{2}\mu (\mathrm{d}x)\right)\mu (\mathrm{d}x)\biggr) \\ &+ 2 \textrm{trace} \biggr(\int_0^\infty \left (f_{12} (x)-\int_0^\infty \dfrac{f_{11} (x) + f_{12} (x)}{2}\mu (\mathrm{d}x)\right)^T \cdot O \cdot \biggr(f_{22} (x)-\int_0^\infty \frac{f_{21} (x) + f_{22} (x)}{2}\mu (\mathrm{d}x)\biggr)\mu (\mathrm{d}x)\biggr) \\
&= 2 \textrm{trace}( UDV^T \cdot O) = 2 trace( D (U^T \cdot O^T \cdot V)^T).
\end{split}
\end{equation}
\endgroup
Now, this is maximized for $O \in SO(2)$, when $O= V^T \begin{bmatrix}
    1 & 0 \\
    0 & \text{det} (V^TU)
  \end{bmatrix}U$. Plugging in this minimizing value for $O$, we get the solution for \\ $\inf_{ (f'_1,f'_2) \in  (\Ocal,\Ccal)_{ (f_{21},f_{22})}} (\rho ( (f_{11},f_{12}), (f'_1,f'_2)))^2 $ as required.
  
 \begin{lemma}
 \label{lemma:invariance}
Assume that $ (f',g') \in \Ccal_{ (f,g)} \cup \Ocal_{ (f,g)} \implies d ( (f',g'), (f,g))=0$. Then,
for all $c \in \mathbb{R}^2,\  O \in SO (2)$,
\begin{eqnarray}
  & & d ( (f_{11},f_{12}), (f_{21},f_{22})) \nonumber\\ & &\hspace{2 em} =d ( (O \odot f_{11} +c, O \odot f_{12} +c), (f_{21},f_{22})).
\end{eqnarray}
 \end{lemma}
 \begin{proof}
 By triangle inequality,
 $d ( (f_{11},f_{12}), (f_{21},f_{22})) \leq d ( (O \odot f_{11} +c, O \odot f_{12} +c), (f_{21},f_{22})) + d ( (O \odot f_{11} +c, O \odot f_{12} +c), (f_{11},f_{12}))$, so by assumption $d ( (f_{11},f_{12}), (f_{21},f_{22})) \leq d ( (O \odot f_{11} +c, O \odot f_{12} +c), (f_{21},f_{22}))$. The lemma follows by considering the reverse inequality.
 \end{proof}
 
\subsection{Obtaining primitives via splines}
\label{Section: Spline}
Our two-step procedure to extract primitives is as follows.

\begin{enumerate}
\item Add change points for each trajectory via the following steps. 
(a) Test whether using the midpoint as a change point reduces the squared error of the fitted polynomial (b) If it does, return the midpoint. (c) Otherwise, test whether using the midpoint of the valid interval of the half with the larger square error as a change point reduces the squared error. Rule out the other half as a site for change points. (d) Repeat (b)-(c) until either a change point is found or no further candidates exist. (e) If a change point was added previously, repeat (a)-(d) for the two segments and any subsequent segments. Stop when no more change points are added.

\item Combine the change points from all trajectories in the following manner. Remove change points via a forward search in the following way. Suppose that we have a set, $\mathcal{C}$, of $L$ ordered change points, $c_1, c_2, ..., c_L$, across all trajectories. Let $c_0$ denote the start point and $c_{L + 1}$ denote the end point. Define $\epsilon$ to be our tolerence. Proceed in these steps:
(a) Set $\ell = 0$, $\ell' = 1$, and $\ell'' = 2$; (b) Fit a polynomial to each trajectory from $c_{\ell}$ to $c_{\ell''}$; (c) If the sum of the squared error of the fitted polynomials is below $\epsilon$ or there are only 4 observations between $c_{\ell}$ and $c_{\ell''}$, remove $c_{\ell'}$ from the set of $\mathcal{C}$. Otherwise, increment $\ell$. Increase $\ell'$ and $\ell''$ by one and go back to (b) if $\ell'' \leq L + 1$; (d) Set $L$ to be the size of $\mathcal{C}$. Re-index the change points in $\mathcal{C}$ from one to $L$ and return $\mathcal{C}$.

To select $\epsilon$ from a set of potential tolerances, we set it to be the value that after running (2), minimizes
\begin{eqnarray*}
\sum\limits_{i = 1}^n \sum\limits_{\ell = 1}^{L + 1} \left(f(t_i) - \hat{f}_\ell(t_i)\right)^2 \mathbbm{1}_{(t_i \leq c_\ell)} + L + 2.
\end{eqnarray*}
\end{enumerate}


\begin{figure*}[!tp]
\centering
\begin{subfigure}{.19\textwidth}
\centering
\includegraphics[width = 1\textwidth]{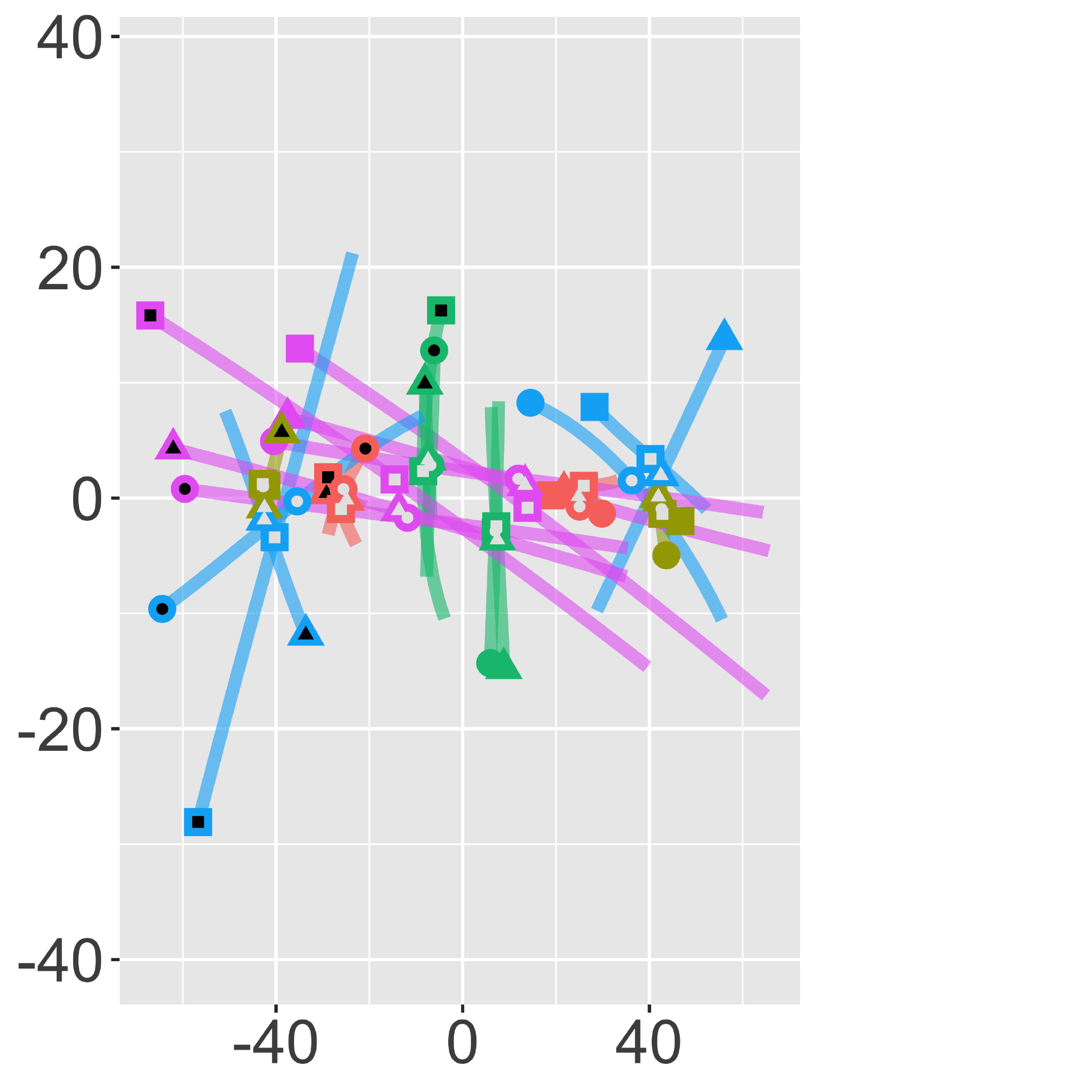}
\caption{Multidim. scaling approach (cf. Section~\ref{sssection:mds})}
\end{subfigure}
\begin{subfigure}{.19\textwidth}
\centering
\includegraphics[width = 1\textwidth]{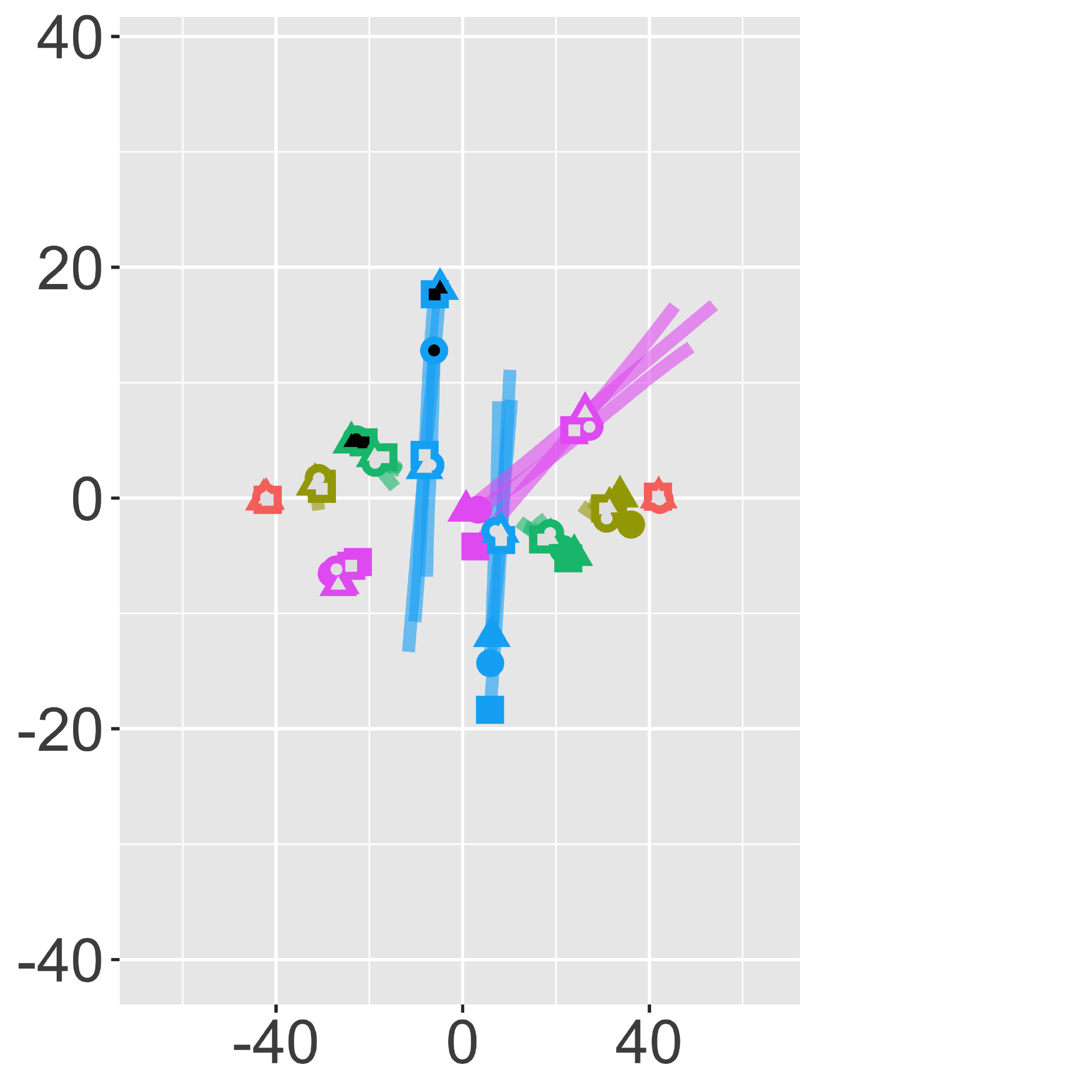}
\caption{First geometric approx. (cf. Section~\ref{sssection:geometric_approx})}
\end{subfigure}
\begin{subfigure}{.19\textwidth}
\centering
\includegraphics[width = 1\textwidth]{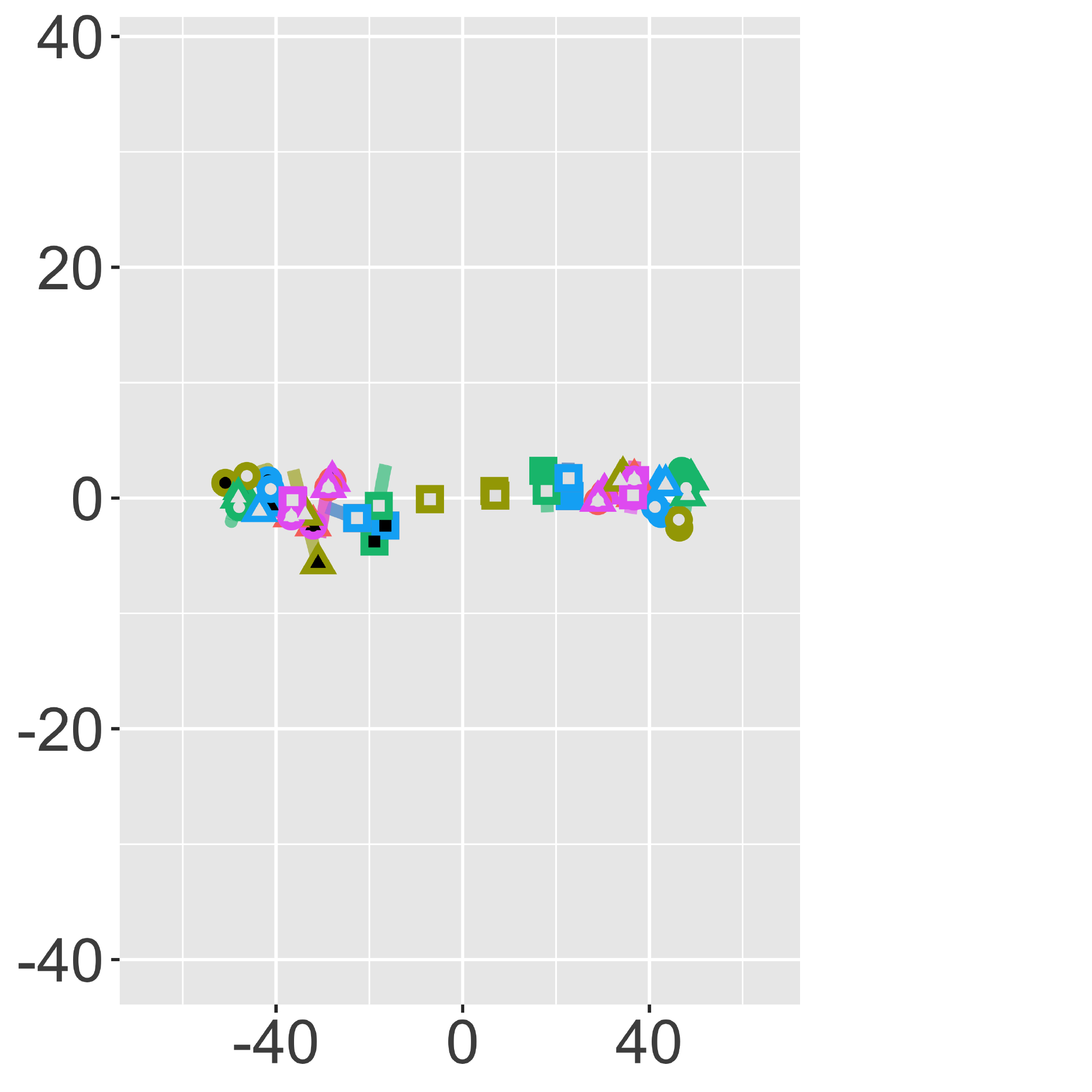}
\caption{Second geometric approx. (cf. Section~\ref{sssection:geometric_approx})}
\end{subfigure}
\begin{subfigure}{.19\textwidth}
\centering
\includegraphics[width = 1\textwidth]{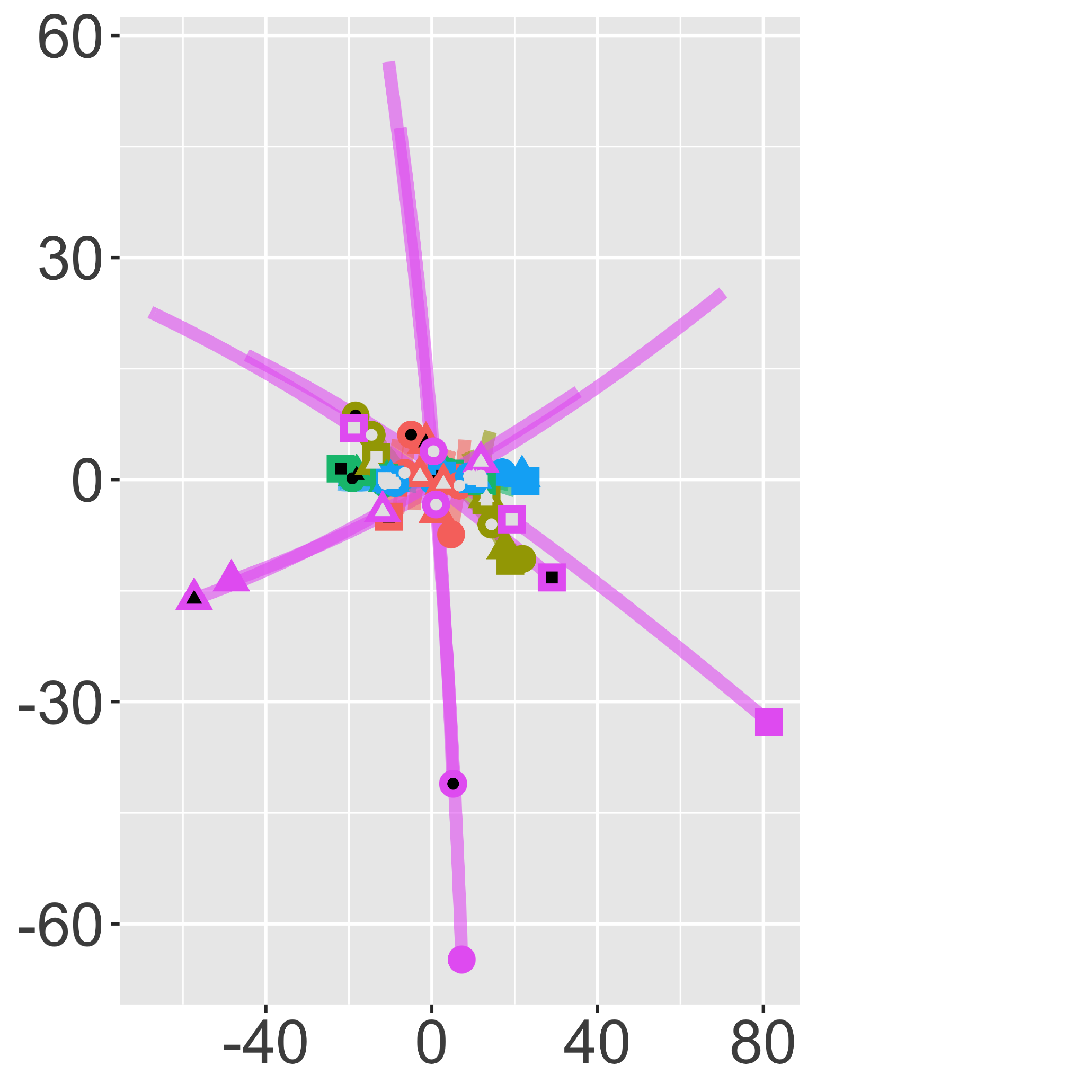}
\caption{Polynomial coefficients (cf. Section~\ref{ssection:clustering_quality})}
\end{subfigure}
\begin{subfigure}{.19\textwidth}
\centering
\includegraphics[width = 1\textwidth]{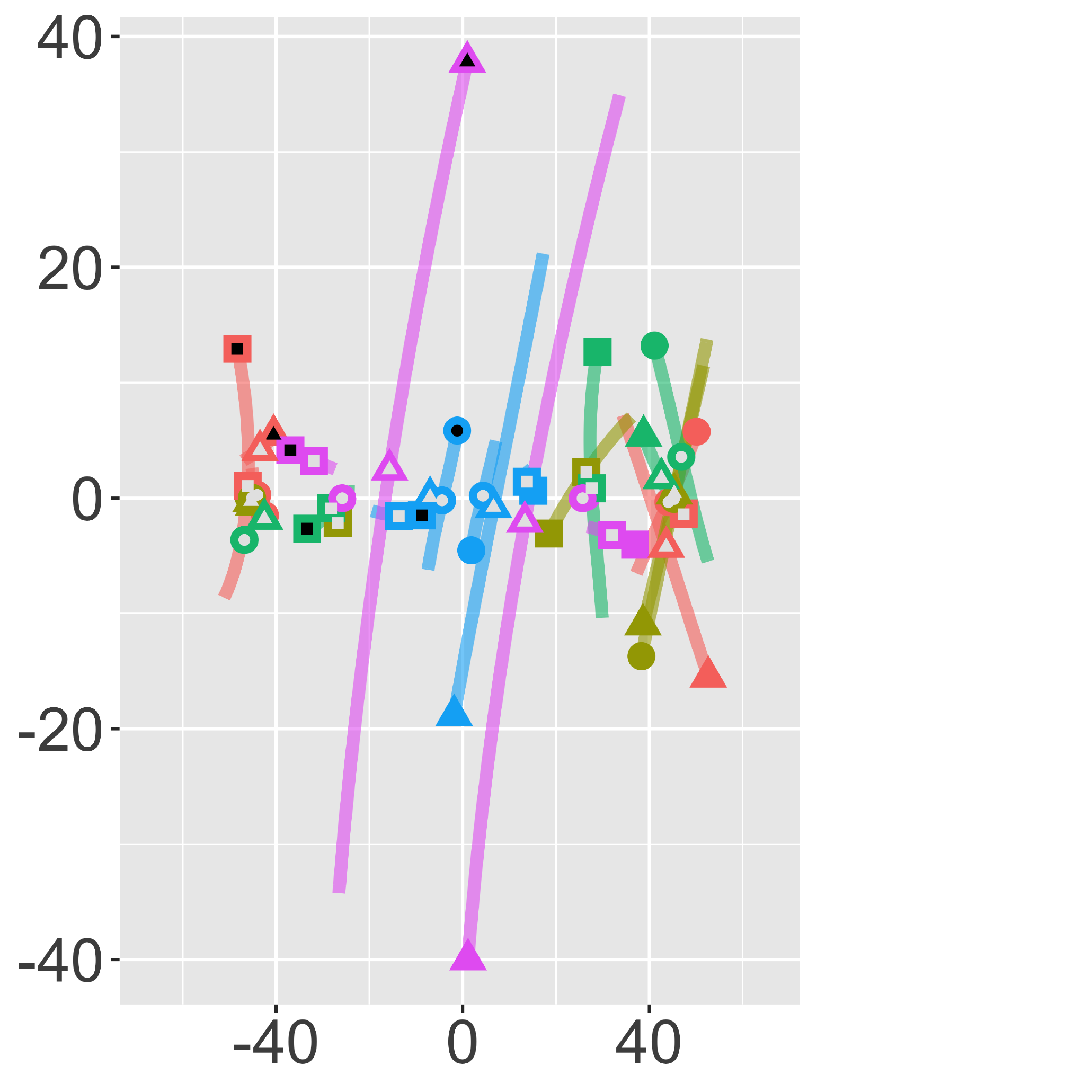}
\caption{DTW cost matrix (cf. \cite{Wang-Zhao-2017})}
\end{subfigure}
\caption{Plot of the three most typical interactions organized from the cluster with the most interaction to the cluster with the fewest for various methods. See Figure \ref{fig:two_step_typical_encounter_plot_approx_mds} for the legend and for how the interactions are oriented.
}
\label{fig:two_step_typical_encounter_plot_approx}
\end{figure*}

\subsection{Algorithm for centering and reorienting primitives}
 \label{ssection:algo_centering}
 
Algorithm~\ref{algo: primitive reorienting} provides a way to reorient one set of interactions to another and is embedded in Algorithms~\ref{algo:first_geom_approx} and~\ref{algo:second_geom_approx}.
 \begin{algorithm}[ht]
\caption{Centering and reorienting interactions}
\label{algo: primitive reorienting}
\algsetup{linenosize=\small}
\scriptsize
Input: Two interaction samples, $(f_{i1},f_{i2})$ and $(f_{j1},f_{j2})$\\ 
Output: A centered $(f_{i1},f_{i2})$ and a centered $(f_{j1},f_{j2})$ reoriented to the centered $(f_{i1},f_{i2})$
\begin{algorithmic}[1]
\STATE  Set $\widetilde{f}_{i}(t) \in \mathbbm{R}^{2t_m} \cross \mathbbm{R}^{2}$ to be the concatenation of $f_{i1}$ and $f_{i2}$ such that $\widetilde{f}_{i}(t) = f_{i1}(t)$ for $t = 1, 2, \dots t_m$ and $\widetilde{f}_{i}(t) = f_{i2}(t)$ for $t = t_m + 1, t_m + 2, \dots 2t_m$ and $\widetilde{f}_{j}(t) \in \mathbbm{R}^{2t_m} \cross \mathbbm{R}^{2}$ be the same concatenation of $f_{i1}$ and $f_{i2}$.
\STATE For $\overline{f}_i(t) \in \mathbbm{R}^{2t_m} \cross \mathbbm{R}^{2}$ and $\overline{f}_i(t) \in \mathbbm{R}^{2t_m} \cross \mathbbm{R}^{2}$, set 
\begin{align*}
    \overline{f}_i(t) &= \widetilde{f}_i(t) - \frac{1}{2t_m} \sum_{t' = 1}^{2t_m} \widetilde{f}_i(t')\\
    \overline{f}_j(t) &= \widetilde{f}_j(t) - \frac{1}{2t_m} \sum_{t' = 1}^{2t_m} \widetilde{f}_j(t').
\end{align*}
\STATE Perform singular value decomposition to get the matrices $U$, $D$, $V$ such that $UDV^T = \overline{f}_j(t)^T \overline{f}_i(t)$. 
\STATE From before, let
\begin{align*}
    \tilde{\Ocal} &= V^T \begin{bmatrix}
    1 & 0 \\
    0 & \operatorname{det} (V^TU)
  \end{bmatrix}U.
\end{align*}
Then, set $\overline{f}'_{i1}(t), \overline{f}'_{i2}(t) \in \mathbbm{R}^{T}\cross\mathbbm{R}^{2}$ to be the matrices such that for $t = 1, 2, \dots, T$,
\begin{align*}
    \overline{f}'_{i1}(t) &= \overline{f}_{i}(t)\\ \overline{f}'_{i2}(t) &= \overline{f}_{i}(t + t_m).
\end{align*}
On the other hand, set $\overline{f}'_{j1}(t), \overline{f}'_{j2}(t) \in \mathbbm{R}^{T}\cross\mathbbm{R}^{2}$ to be the matrices such that for $t = 1, 2, \dots, T$,
\begin{align*}
    \overline{f}'_{j1}(t) &= (\tilde{\Ocal}\overline{f}_{j})(t)\\
    \overline{f}'_{j2}(t) &= (\tilde{\Ocal}\overline{f}_{j})(t + t_m).
\end{align*}
\STATE Return $(\overline{f}'_{i1}, \overline{f}'_{i2})$ and $(\overline{f}'_{j1}, \overline{f}'_{j2})$.
\end{algorithmic}
\end{algorithm}

\end{document}